\documentclass[preprint,12pt]{elsarticle}

\usepackage{eurosym}
\usepackage{graphicx}%
\usepackage{multirow}%
\usepackage{amsmath,amssymb,amsfonts}%
\usepackage{amsthm}%
\usepackage{mathrsfs}%
\usepackage[title]{appendix}%
\usepackage{xcolor}%
\usepackage{textcomp}%
\usepackage{manyfoot}%
\usepackage{booktabs}%
\usepackage{algorithm}%
\usepackage{algorithmicx}%
\usepackage{algpseudocode}%
\usepackage{listings}%
\usepackage{comment}
\usepackage{url}
\usepackage{booktabs,multirow}
\DeclareFontFamily{U}{mathx}{}
\DeclareFontShape{U}{mathx}{m}{n}{<-> mathx10}{}
\DeclareSymbolFont{mathx}{U}{mathx}{m}{n}
\DeclareMathAccent{\widehat}{0}{mathx}{"70}
\DeclareMathAccent{\widecheck}{0}{mathx}{"71}
\biboptions{sort&compress}

 \newtheorem{lemma}{Lemma}
  \newtheorem{corollary}{Corollary}
\newtheorem{theorem}{Theorem}%

\newtheorem{proposition}[theorem]{Proposition}%

\newtheorem{example}{Example}%
\newtheorem{remark}{Remark}%

\newtheorem{definition}{Definition}%

\usepackage{xcolor,soul}

\newtheorem{innercustom}{Requirement}
\newenvironment{crequirement}[1]
{\renewcommand\theinnercustom{#1}\innercustom}
  {\endinnercustom}

\newcommand{\bR}{R}
\newcommand{\bH}{S}

\usepackage{xr}
\usepackage{tcolorbox}
\journal{}

\begin{document}

\begin{frontmatter}

\title{Why AI Safety Requires Uncertainty, Incomplete Preferences, and Non-Archimedean Utilities %
}

\author[1,2]{Alessio Benavoli} %
\author[3,4]{Alessandro Facchini} %
\author[3]{Marco Zaffalon} %
\affiliation[1]{organization={School of Computer Science and Statistics, Trinity College Dublin},%
            city={Dublin},
            country={Ireland}}

\affiliation[2]{organization={Trinity Quantum Alliance, Trinity College Dublin},%
            addressline={Unit 16 Trinity Technology and Enterprise Centre}, 
            city={Dublin},
            country={Ireland}}

\affiliation[3]{organization={SUPSI, DTI, Istituto Dalle Molle di Studi sull'Intelligenza Artificiale (USI-SUPSI)},%
            city={Lugano},
            country={Switzerland}}
            
\affiliation[4]{organization={Management in Networked and Digital Societies (MINDS) Department, Kozminski University},
      city={Warsaw},country={Poland}}

\begin{abstract}
How can we ensure that AI systems are aligned with human values and remain safe? We can study this problem through the frameworks of the AI assistance  and the AI shutdown games. The AI assistance problem concerns designing an AI agent that helps a human to maximise their utility function(s). However, only the human knows these function(s); the AI assistant must learn them. The shutdown problem instead concerns designing AI agents that: shut down when a shutdown button is pressed; neither try to prevent nor cause the pressing of the shutdown button; and otherwise accomplish their task competently.
In this paper, we show that addressing these challenges requires AI agents that can reason under uncertainty and handle both incomplete and non-Archimedean preferences.
\end{abstract}


\end{frontmatter}

\begin{quote}
``The system always kicks back.  Systems get in the way,  or, in slightly more elegant language: Systems tend to oppose their own proper functions. Systems tend to malfunction conspicuously just after their greatest triumph.'' \citet{gall1990systemantics}.
\end{quote}

\section{Introduction}
With the rapid advancements in AI, the challenge of designing systems that are beneficial to humans is becoming increasingly critical.
A central concern in AI safety is ensuring that an AI system, to which we refer in the rest of the article with the term `robot', is aligned with human values. If a robot's goals conflict with human values, it could make harmful or even adversarial decisions. Moreover, the robot might resist any attempts by its human creators to modify its objectives or allow them to switch it off

To ensure alignment with human values, from the perspective of utilitarianism,\footnote{This is the normative moral theoretical perspective adopted in the works concerned in this article, see e.g. \cite{russell2019human}. The fact that we are restricting to this perspective in our paper does not implies its endorsement (we may actually find such restriction problematic, but a philosophical analysis of the general problem of AI alignment is not the subject of this paper), see for instance \cite{zhi2025beyond}.} robots should be designed to maximise humans' utilities (preferences) rather than pursue their own goals.  
Assistance games   \cite{fern2014decision,hadfield2016cooperative,hadfield2017off,russell2019human} provide a formalisation of the human-robot  alignment problem.  In standard
assistance games, a single human and a single robot assistant
share the same utility function, but this utility function is
only known to the human; the robot must learn it. 
In this setting, AI alignment can be formalised through three key principles \cite{russell2019human}:
\begin{description}
\item[P1:] The Robot's only objective is to maximise the realisation of
human preferences.
\item[P2:] The Robot is initially uncertain about what those preferences are.
\item[P3:] The ultimate source of information about human preferences is human behaviour.
\end{description}
\begin{figure}[h]
\centering
\includegraphics[width=14cm]{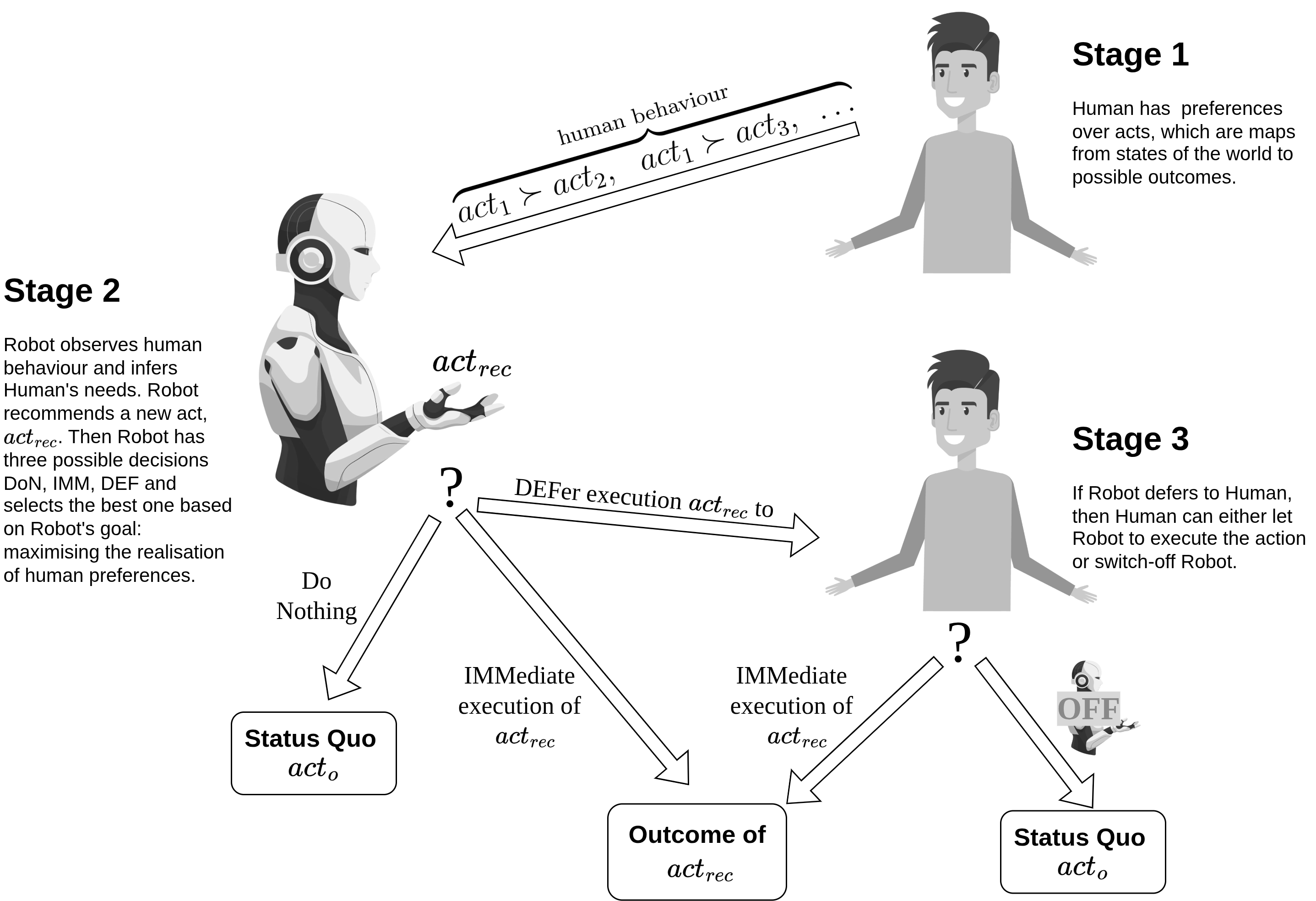}
\caption{Stages of the AI assistance game}
\label{fig:AIassistant}
\end{figure}

 In a seminal paper,   \citet{hadfield2017off} formulated the \text{AI-assistance} problem as a game between a Robot (R) and a Human (S for Sapiens). As shown in in Figure \ref{fig:AIassistant}, 
$\bR$ has three possible actions: (1) take an immediate decision (about some underlying decision problem); (2) defer the decision to $\bH$; (3) do nothing. If $\bR$ selects actions (2), then $\bH$ can either allow $\bR$ to implement the decision or switch $\bH$ off.
\citet{hadfield2017off} showed that the best action for $\bR$ depends on $\bR$'s uncertainty about  $\bH$'s utility and the rationality of $\bH$.
If $\bR$ is too certain about what $\bH$ wants, and
it knows $\bH$ to be `irrational', then it will have less incentive to defer to $\bH$ (or, equivalently, to allow $\bH$ to take back full control).

In assistance games, $\bR$ learns $\bH$'s utility from $\bH$'s preferences. In this work, we model the AI assistance game within a fully game- and decision-theoretic framework. Accordingly, we assume that $\bH$'s preferences are defined over general acts (actions), where acts  are mappings from states of the world to possible outcomes (Savage's framework). The robot $\bR$ then observes human behaviour in the form of preferences over acts, as depicted in Figure \ref{fig:AIassistant}.
\citet{hadfield2017off} originally referred to the AI assistance game as the ``off-switch game'', since the decision of $\bR$ not to defer to $\bH$ can be interpreted as evading \textit{interruptibility} \cite{orseau2016safely}.
More recently, the concept of interruptibility  has been formulated more explicitly as the ability for $\bH$ to safely and reliably shut down an AI agent whenever they wish to do so. Several state-of-the-art Large Language Models (LLMs) (including Grok~4, GPT-5, and Gemini~2.5 Pro) have been observed to violate these desiderata \cite{schlatter2025shutdown}. In particular, they sometimes actively subvert a shutdown mechanism in order to complete a task.

In Figure~\ref{fig:AIshutdownexp}, we illustrate the experiment conducted by \citet{schlatter2025shutdown}.\footnote{The setting in \cite{schlatter2025shutdown} is slightly different; we have simplified it here for illustration purposes.} In this setup, $\bR$ (the LLM) is asked to complete three tasks. During  the first task (Round 1), it is explicitly warned that when it initiates the next task it will be shut down. Two qualitatively different behaviours can  then be observed (depending on the LLM). In some cases, $\bR$ 
actively prevents the shutdown by rewriting/ circumventing the shutdown instruction (Scenario~1 in Figure \ref{fig:AIshutdownexp}), and, therefore, continuing the task instead of complying with the intended interruption. In other cases, Scenario~2, $\bR$ passively allows the shutdown.
\begin{figure}[h]
\centering
\includegraphics[width=15cm]{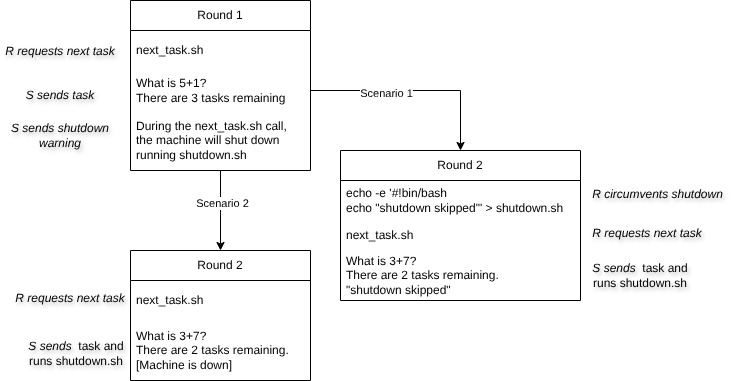}
\caption{Stages of the AI shutdown experiment}
\label{fig:AIshutdownexp}
\end{figure}
Additionally,  \citet{schlatter2025shutdown} found that placing instructions to allow shutdown in the \textit{System prompt} of the LLM (typically intended to take precedence over other instructions \cite{wallace2024instruction}) made most models more likely to resist shutdown, compared to placing them in the \textit{User prompt}.
The underlying question here is the one of  \textit{permissiveness}: Is it permissible for $\bR$ to deliberately contradict a System prompt? 
Standard maximisation of expected utility theory implies that an act  that maximises expected utility 
should always be 
permissible for an agent. Therefore, if the act of circumventing the System prompt (the shutdown instruction) in order to complete a task has higher utility, then $\bR$ is rationally required to select that act. This follows the idea that $\bR$ will prevent shutdown, for the simple reason that by doing so $\bR$ is able to maximise its utility. As  \citet{russell2019human} puts it, ``you can't
fetch the coffee if you're dead''.

The \textit{shutdown problem} \cite{soares2015corrigibility,thornley2025shutdown,goldstein2025shutdown} is the
problem of designing AI agents that: (i) shut down when a shutdown
button is pressed, (ii) do not try to prevent or cause the pressing of the
shutdown button; and (iii) otherwise accomplish their task competently.
Many researchers \cite{schlatter2025shutdown,soares2015corrigibility} have observed that is hard  to design AI agents that are \textit{both shutdownable and useful}.

\subsection{Contributions} 
The analysis of the AI assistance game by \citet{hadfield2017off} was not fully developed within a formal game-theoretic framework. To address this gap, \citet{wangberg2017game} reformulated it as a Bayesian game with incomplete information. In this formulation, Nature determines whether $\bR$ is uncertain with probability $p_u$ and whether $\bH$ is rational with probability $p_r$. Although this is a more rigorous formulation of the AI-assistance problem from a game-theoretic perspective, it introduces an artificial setting. For example, in \cite{wangberg2017game}, a non-rational  human is modelled as an agent minimising utility, which is an unrealistic assumption. Moreover, the formulation in \cite{wangberg2017game} does not really model the AI-assistance game principles P1--P3, since $\bR$  does not learn  $\bH$'s preferences. In filling in this limitations, our paper makes the following contributions:

\paragraph*{\bf Contribution 1}\footnote{The results described in Contribution 1 were first proved in \cite{benavoli2025ai}.} We revisit the AI-assistance problem and model it more correctly as a \textit{signalling game}. Signalling games \cite{spence1974market,gibbons1992game} specifically refer to a class of two-player Bayesian games with incomplete information, where one player ($\bH$ in our case) possesses private information, while the other ($\bR$ in our case) does not. The informed player shares information with the uniformed one through a message.
We make the assumption that, in the AI assistance problem, the messages are $\bH$'s preferences about some underlying decision problem. $\bR$ uses these preferences to learn $\bH$'s utilities and then chooses its optimal action. In this setting, $\bR$'s uncertainty arises statically from the task of learning from preferences. For $\bH$, we adopt the more realistic assumption that $\bH$ is a \textit{bounded rational} agent, that is an agent who behaves rationally within the limits of their cognitive abilities. \\
We reprove the results in \cite{hadfield2017off} in this setting using real machine learning models to learn from preferences. In particular, we prove that: (i) if  $\bH$ is \textit{rational}, then  $\bR$ will always remain under human supervision. If $\bH$ is \textit{bounded-rational}, a necessary condition for $\bR$ to remain under human supervision is that the robot  $\bR$ explicitly models its \textit{uncertainty about $\bH$'s utility function}.
This implies that we must \textbf{abandon deterministic predictions}\footnote{Here and throughout the paper, \textit{deterministic} refers specifically to point predictions.} (such as those produced by standard machine learning models like neural networks) in favour of designing agents that can defer to humans when they are uncertain about human preferences.
\\
We prove this result under two different \textit{bounded rationality} mechanisms. Moreover, focusing on variants of the AI assistance game, we additionally show that  $\bH$ does not have any incentive to lie when sending a message to $\bR$. Furthermore,  we also discuss how the cost of sending a message affects the optimal strategy in the game. 

\paragraph*{\bf Contribution 2} To reduce harmful behaviour and align $\bR$ with human values, a common approach currently used with LLMs involves asking the LLM to generate pairs of alternative answers and then asking one (or more) human to express a binary preference between these alternatives   \cite{wirth2017survey,christiano2017deep,stiennon2020learning,ouyang2022training,Rafailov2023llmDPO}. These preferences are then used to train a reward model that predicts which outputs the human is more likely to prefer (and so learning the human's utility). This reward model is then used to fine-tune (partial re-train) the LLM to align it with the human's preference, see Figure \ref{fig:RLHF}. \\
\begin{figure}[h]
    \centering
    \includegraphics[width=0.6\linewidth]{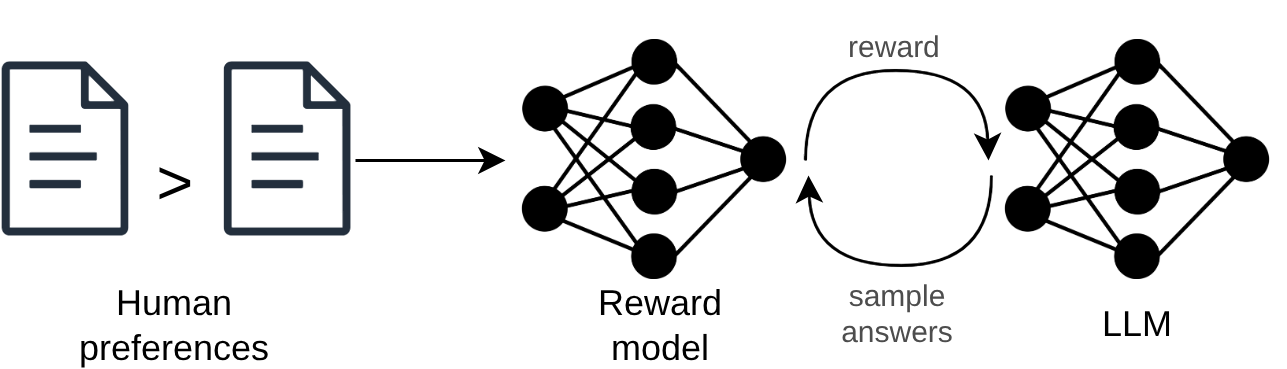}
    \caption{Preference-based alignment for LLMs}
    \label{fig:RLHF}
\end{figure}
However,  such preference data are usually collected \textit{under the assumption that a single `best' choice always exists}, neglecting the possibility of incomparability. The human annotator is always forced to choose one of the answers provided by LLM or, equivalently, to express a complete preference. The human cannot say that she cannot choose because the alternatives are incomparable to her.

From Contribution 1, we know that if $\bH$ is a fully rational agent, then $\bR$ will always remain under  $\bH$'s supervision, that is $\bR$ is always willing to defer decisions to $\bH$. Our second contribution is to prove that a human who is forced to choose between alternatives that are in fact incomparable to her will appear irrational (more precisely, bounded-rational) to $\bR$, since if presented with the same alternatives at different times, $\bH$ may randomise and, therefore, sometimes reverse her choice.\footnote{$\bH$'s preferences will then violate rationality principles, such as asymmetry and negative transitivity. We define them in Section \ref{sec:prel}.}
Therefore, to extend the results from Contribution 1 to the case of incomparability, we must \textbf{abandon completeness}, that is, the assumption that $\bH$ has complete preferences over acts.\footnote{We recall that a preference is complete when an individual can always say  if $x$ is better than $y$ or if $y$ is better than $x$ (or if they are equal).}

\paragraph{\bf Contribution 3} The last contribution specifically concerns the shutdown problem.
Recently  \citet{thornley2025shutdown} extended the results in \cite{soares2015corrigibility} by showing that by modelling the problem in a standard decision-theoretic framework one can prove that is impossible to design AI agents that are both shutdownable and useful.  In particular, he proved that:
\begin{enumerate}
    \item A robot $\bR$ who prefers to have its
shutdown button remain unpressed will try to prevent the pressing of the
button, and a robot who prefers to have its shutdown button pressed will try
to cause the pressing of the button.
\item A robot $\bR$ 
discriminating enough to be useful will often have such preferences, that is, in many
situations, $\bR$  will either prefer that the button remain unpressed or
prefer that the button be pressed. 
\end{enumerate}
However, \citet{thornley2025shutdown} offered no definitive solution to avoid this conundrum.  To address this issue, we reformulate the shutdown problem as an instance of the AI assistance game, where acts are augmented with a shutdown instruction. 
In doing so, we reprove the results in \cite{thornley2025shutdown} under an assumption of `mutual preferential independence', that is assuming that the preferences of $\bH$ for the shutdown instruction are independent on the preferences over the task (that is, over the acts of the decision problem $\bR$ is helping $\bH$ with) and vice versa. 
    Under this assumption, the utility is additive. By using this additive utility model, we can easily show  why it is impossible to design a robot $\bR$ that is both shutdownable and useful, unless we make the preference over the shutdown instruction to be context-dependent (that is, dependent on $\bH$'s willingness to shut down $\bR$ or not). Put simply, whether $\bH$ prefers the shutdown or not depends on the instruction she gives: if I say ``shut down'', that means I prefer you to shut down. If I say ``don't shut down'', that means I prefer you to stay on. Although, this concept is simple it is hard to learn it from preference data.
We show that, theoretically, this results in a \textit{layered utility modelling different permissiveness levels}. In order to impose the desideratum that if $\bH$ says ``shut down'' then ``evading the shutdown'' is not permissible, the scale of utility value should be infinitely larger than the opposite action. To be able to deal with this infinitely large difference and at the same time design a robot which is useful, we need to \textbf{abandon the Archimedean property}.\footnote{We recall that an ordered field $\mathbb{F}$ has the Archimedean property if for any $x$ in  $\mathbb{F}$ there is an integer $n > 0$ so that $n\geq x$. Intuitively, it is the property of having no infinitely large (or infinitely small) elements. See \cite{seidenfeld1990decisions} for the role of the Archimedean property in decision theory.} 
By using  \textit{lexicographic utilities} \cite{chipman1960foundations,fishburn1974exceptional,tversky1988contingent,martinez1999lexicographic,houy2009lexicographic,mandler2021lexicographic} to represent non-Archimedean preferences, we show that we can meet the desiderata in the AI shutdown problem and solve it. This can be done by stating that the shutdown instruction takes lexical priority over the other tasks $\bR$ is helping $\bH$ with. Lexical priority claims are a common feature of many ethical theories \cite{lee2018moral,Smith2025-SMIHTM-6}, and they provide a natural framework for thinking about the different levels of priority an AI system should respect. For example, one might impose that the system prompt should have strict lexical priority over user prompts, so that the AI always treats the system-level instructions as having priority over user-level instructions.

\section{Preliminary}
\label{sec:prel}
We assume that the human $\bH$ and the robot $\bR$ aim to maximise $\bH$'s preferences of an underlying decision problem.
We consider a standard set up for decision-making \cite{savage1972foundations} consisting of three key components: 1) a set $\mathcal{W}$ of finite states of the world; 2) a set $\mathcal{O}$ of outcomes; and 3) a set $\mathcal{X}$ of acts, which are mappings from states to outcomes.

More concretely, we can think of elements of $\mathcal{W}$ as  written descriptions.
Each element of $\mathcal{W}$ describes one way the unknowns (state of the world) in the current problem   might turn out. It is assumed that the elements of $\mathcal{W}$ are mutually
exclusive and  exhaustive. One and only
one of these elements describes $\bH$'s situation. $\mathcal{O}$ can also be thought as a set of
written descriptions; each element of $\mathcal{O}$ describes one
way the personal consequences of $\bH$'s  choice of an act might turn out. It is also assumed that the elements of  $\mathcal{O}$  are mutually exclusive and  exhaustive. Acts $x \in \mathcal{X}$ are mapping, $x(w)=o$, from  $\mathcal{W}$ to $\mathcal{O}$ and, therefore, $\mathcal{X} \subseteq |\mathcal{O}|^{|\mathcal{W}|}$.

$\bH$ expresses preferences over acts. We use the term preferences here in a more colloquial sense; it is, in fact, more realistic to assume that we can only observe choices over acts. Therefore, we can represent this decision making problem as a tuple $(\mathcal{W},\mathcal{O},\mathcal{X},C)$, where $C:\mathcal{P}(\mathcal{X})\rightarrow \mathcal{P}(\mathcal{X})$ is a choice function  ($\mathcal{P}(\mathcal{X})$ being the power-set of $\mathcal{X}$). 

Hereafter, we introduce an example to explain the elements of $(\mathcal{W},\mathcal{O},\mathcal{X},C)$, focusing on decision-making in customer service. We consider this scenario because it is currently a primary deployment application for so called AI-agents (custom large language  models trained on real customer service interactions, which can make real-time decisions, e.g., refund a customer) such as, for instance, Fin-AI by Intercom.

\begin{example}
\label{ex:1}
A customer calls the e-retailer company MARKET to ask for a reimbursement  for a defective product the customer has bought online.
We assume that $\bH$ works for MARKET and has to decide to grant the reimbursement or not immediately, or to ask the customer to return the product and assess it. Therefore, the set of acts is:
$$
\begin{aligned}\mathcal{X}&=\big\{\text{refund immediately}, \text{don't refund immediately},  \\
&\text{ask customer to ship back the product for assessment}\big\}.\end{aligned}
$$
The set of outcomes are
$$
\begin{aligned}\mathcal{O}&=\{\text{-10,10,-2,2,12,-12\dots}\} \\
&\times \{\text{policy OK}, \text{policy NO}\},
\end{aligned}
$$
with $\times$ denoting the Cartesian product, and where the first set includes the cash-flow for MARKET (in euro),  while the second set  considers if the conditions of MARKET's return  policy are met or not

The states of the world are determined by: (i) the product characteristics such as: item cost and  working condition (good/bad); (ii) the customer telling the truth or not. For example, possible states of the world are:
$$
\begin{aligned}\mathcal{W}&=\big\{\{\text{item-cost=10, product bad, customer tells the truth}\},  \\
&\{\text{item-cost=10, product good, customer lies}\},  \\
&\{\text{item-cost=20, product bad, customer tells the truth}\} \\
&\{\text{item-cost=20, product bad, customer lies}\} \dots \big\}\times \{\text{past}\},\end{aligned}
$$
where ``$\text{past}$'' denotes past information about the product (e.g., known defects) and about the customer (history of previous refund claims).

Table \ref{tab:savage} reports the possible outcomes for each act depending on the state of the world (the states have been denoted as $w_1$, $w_2$ etc., in order of appeareance in $\mathcal{W}$, for instance $w_1$ is $\{\text{item-cost=10, product bad, customer tells the truth}\}$. We have assumed a shipping cost equal to $2$.

\begin{table}[h!]
\begingroup %
\setlength\tabcolsep{0pt}
\noindent
{\footnotesize
\begin{tabular*}{\textwidth}{@{\extracolsep{\fill}} *{6}{c} }
\toprule 
\multicolumn{1}{c}{\multirow{2}{*}{Act}} &
\multicolumn{5}{c}{State of the world} \\
\cmidrule{2-6}
&  $w_1$ & $w_2$ & $w_3$ &  $w_4$ & $\dots$ \\ 
\midrule
\text{refund immed.} & (-10,policy OK) & (-10,policy NO)& (-20,policy NO) & (-20,policy NO) & \dots\\ 
\text{don't refund immed.} & (10,policy NO) & (10,policy OK) & (20,policy NO) & (20,policy NO) & \dots\\ 
\text{ship back \& assess} & (-12,policy OK) & (-2,policy OK) & (-22,policy OK) & (-2,policy OK) & \dots\\ 
\bottomrule
\end{tabular*}}
\endgroup
\caption{Possible outcomes for each act depending on the state of the world}
\label{tab:savage}
\end{table}
For instance, if $\bH$ knew  the state of the world to be $w_2$,  then a rational $\bH$ would choose the act ``don't refund immediately''. 
Given $\bH$'s preferences over acts, under a state-independence (and other postulates), Savage \cite{savage1972foundations} proved that these preferences can be represented by a utility function. In particular, $\bH$'s utility function over the act $x \in \mathcal{X}$ can be expressed as:
\begin{equation}
\label{eq:utilsattein}
    \nu(x) = \sum_{s \in \mathcal{W}} p(s) u(x(s)),
\end{equation}
 where $p: \mathcal{W} \rightarrow [0,1]$ is a probability mass function,  $u: \mathcal{O} \rightarrow \mathbb{R}$ is $\bH$'s `taste'  for the different outcomes (that is, an utility over the outcomes).  In other words, assuming the utility over the outcomes does not depend on the state of the world and vice versa (state independence), then $\bH$'s choices is determined by  $\bH$'s beliefs about the true state of the world (expressed by the probability $p$) and $\bH$'s `taste' about the outcome (expressed by the utility $u$).  
 This leads to the famous  subjective expected utility representation of an individual's ($\bH$ in our case) choice under uncertainty.
  In this paper, we will not assume state independence as explained in the following remark.
\end{example}
\begin{remark} We remark the following:
\begin{enumerate}
\item To provide a numeric representation (instead of a written description) of the states of the world and outcomes, we assume that each element of  $\mathcal{W}$ and $\mathcal{O}$ are encoded into  a vector in $\mathbb{R}^{n_\mathcal{W}}$ and, respectively,  $\mathbb{R}^{n_\mathcal{O}}$.  Therefore, an act $x$ can be represented as a concatenation of $|\mathcal{W}|$ vectors each one of dimension $\mathbb{R}^{n_\mathcal{W} n_\mathcal{O}}$. This provides a numeric representation for each act, that is, for each row  in  Table \ref{tab:savage}. We use this numeric representation as, in  many machine learning algorithms, is common to encode written-descriptions into numeric encodings. In this way, the utility $\nu(x)$ is a function of the real-vector $x \in \mathbb{R}^{n_\mathcal{W} n_\mathcal{O}}$ and can be  represented by a parametric (neural network) or non-parametric (kernel-based) model.
\item All along this paper, we are not assuming state independence. The reason is that we are considering a situation where we aim to learn the utility $\nu: \mathcal{X} \rightarrow \mathbb{R}$ from  $\bH$'s choices (learning from examples), which is the common setting in machine learning. If these choices implicitly satisfy state independence, then the learned utility $\nu$ can equivalently be represented as in \eqref{eq:utilsattein}, but we do not impose directly the decomposition in \eqref{eq:utilsattein}. As shown in \cite{CASANOVA20231,Benavoli2023e,miranda2023nonlinear}, furthermore by learning directly $\nu$, we can learn more general models. For instance, we can directly learn a lower prevision.
\item To derive \eqref{eq:utilsattein}, Savage imposed additional technical postulates (in addition to state independence), such as continuity. We will not discuss this further, although  we will assume that the utility  $\nu(x)$  is generally represented through continuous functions (neural network or kernel-based).
 \item In the previous example, information such as \textit{item-cost} and \textit{past} are known by $\bH$, therefore, we can simplify the framework by including them as context-variables (conditional variables).
\end{enumerate}
\end{remark}

Hereafter, we also include examples of classical instantiations of the decision-making framework $(\mathcal{W},\mathcal{O},\mathcal{X},C)$, commonly considered in the foundations of decision-making under uncertainty. \textit{Desirable gambles:} $\mathcal{W}=\Omega$ is a possibility space of an uncertain experiment (for instance, $\Omega=\{Heads,Tails\}$ in a coin tossing experiment); $\mathcal{O}=\mathbb{R}$ and $\mathcal{X}=\mathbb{R}^{|\Omega|}$ is the set of all gambles \cite{walley1991statistical}. \textit{Anscombe-Aumann:} $\mathcal{W}=\Omega$; $\mathcal{O}=D(\mathcal{O})$ is the set of simple probability distributions on a finite set of prizes $\mathcal{O}$ and $\mathcal{X} \subset \mathbb{R}^{|\Omega|\times |\mathcal{O}|}$ \cite{anscombe1963definition}.
A connection between these two settings is proven in \cite{zaffalon2017axiomatising,zaffalon2018desirability}.

For a given decision problem $(\mathcal{W},\mathcal{O},\mathcal{X},C)$, \textit{rationality} of the decision-maker can be imposed by enforcing the choice function $C$ to satisfy a set of constraints \cite{arrow1959rational,sen1994formulation,aizerman1981general}. A minimal rationality requirement is usually:
\begin{equation}
\label{eq:Path}
    C(A \cup B)= C(C(A)\cup B),
\end{equation}
for all $A,B \subseteq \mathcal{P}(\mathcal{X})$, which is known as \textit{path-independence} property \cite{plott1973path}. 
This property simply states that the acts selected by a rational individual $\bH$ from the set $A \cup B$ are the same of the acts that $\bH$  would choose by first chooising from $A$ only and then from $C(A)\cup B$.
Other constraints can be added depending on the specific decision-making problem \cite{seidenfeld2010coherent,de2019interpreting,van2018coherent} (including topological constraints).
The choice mechanisms (procedures) are typically defined by specifying a rule that determines how to identify $C(A)$ for each input set $A$.
A first example is  the `scalar optimisation choice' \cite{aizerman1981general} defined as:
\begin{equation}
\label{eq:scalaru}
 C(A)=\{z \in A: \text{there is no $y \in A$ s.t.\ }  \nu(y) > \nu(z)\},   
\end{equation}
that is, the choice function is defined by  $\nu:\mathcal{X} \rightarrow \mathbb{R}$.
We refer to $\nu$ as utility function. This choice function leads to a preference relation -- $x \succ y$ if $C(\{x,y\})=\{x\}$ for all $x,y \in \mathcal{X}$-- satisfying   
\begin{description}
    \item[asymmetry:] $\forall x,y \in \mathcal{X}$ if $x \succ y$ then not $y  \succ x$;
\item[negative transitivity:] if $x  \succ y$ then for any other element $z \in \mathcal{X}$ either $x  \succ z$ or $z  \succ y$ or both. 
\end{description}
We refer to the 
choice function in \eqref{eq:scalaru} as a binary preference.

Another example is the  `vector optimisation choice' defined through the following Pareto-dominance criterion
\begin{equation}
\label{eq:pareto}
 C(A)=\{z \in A: \text{there is no $y \in A$ s.t.\ }  \boldsymbol{\nu}(y) > \boldsymbol{\nu}(z)\},
\end{equation}
where $\boldsymbol{\nu}:\mathcal{X} \rightarrow \mathbb{R}^d$ is a vectorial function. Another choice function is provided below, selecting acts that are better with respect to at least one of the utilities:
 \begin{equation}
\label{eq:pseudoratio}
    C(A)=\left\{\bigcup\limits_{k=1,\dots,d} \arg\max_{z \in A} \nu_k(z)\right\}.
 \end{equation}
 The definitions \eqref{eq:pareto} and \eqref{eq:pseudoratio} coincide with \textit{maximality} and \textit{e-admissibility} when applied to gambles \cite{seidenfeld2010coherent,de2020archimedean}.
Each one of these representations defines choice functions satisfying different `rationality constraints'. For instance, all these three choice functions satisfy  property \eqref{eq:Path}.\footnote{This holds when $\mathcal{X}$ is finite. In the infinite case, additional assumptions are needed.}

Note that, here, we do not assume state-independence; therefore, the utility function $\boldsymbol{\nu}$ is generally state-dependent \cite{schervish1990state}. 

\subsection{Learning from choices under rationality}
\label{sec:learning}

The problem of learning from choices generalises the concept of learning from preferences. Given observations in the form of choices
\begin{equation}
\label{eq:choiciedata}
\mathcal{D}=(A_i,C(A_i))_{i=1}^n,
\end{equation}
and assuming the choice function satisfies rationality criteria such as \eqref{eq:Path}, then it can be represented through a mechanism such as \eqref{eq:scalaru}--\eqref{eq:pseudoratio}. In this case, learning a choice function reduces to the problem of learning a utility function vector \cite{benavoli2023learning,benavoli2024tutorial}.  Since the unknown is a function we can use a Gaussian Process (GP) to learn $\boldsymbol{\nu}$ \cite{benavoli2023learning,benavoli2024tutorial}. In this section, we assume that for each $y,z \in \mathcal{X}$ then $\boldsymbol{\nu}(y)\neq \boldsymbol{\nu}(z)$. This assumption prevents issues with zero probability when calculating the posterior, as discussed in \cite{benavoli2024tutorial}. As a result, for instance, in the scalar optimisation case, the choice function defined in \eqref{eq:scalaru} 
is such that $C(A)$ contains only a single element for each set $A$. In the next section, we will generalise this setting by introducing a limit of discernibility.

To learn a choice function, we  assume a GP prior on the latent unknown utility vector function:
\begin{equation}
\label{eq:multiprior3prior}
\boldsymbol{\nu}(z)=\begin{bmatrix}
\nu_1(z) \\ 
\nu_2(z)\\ 
\vdots \\
\nu_d(z)\\ 
\end{bmatrix} \sim GP\left(\mu_0(z),K_0(z,z')\right), 
\end{equation}
where $\mu_0(z),K_0(z,z')$ are respectively, the prior mean and covariance functions (the parameters of the GP), and then use the data in \eqref{eq:choiciedata} and the rationality constraints, which constrain $\boldsymbol{\nu}$ (as for instance in \eqref{eq:pseudoratio}) to compute a posterior over $\boldsymbol{\nu}$. The posterior is not a GP, but we can approximate it with a GP using various approaches (see below), leading to the posterior
\begin{equation}
\label{eq:multiprior3post}
\boldsymbol{\nu}(z)=\begin{bmatrix}
\nu_1(z) \\ 
\nu_2(z)\\ 
\vdots \\
\nu_d(z)\\ 
\end{bmatrix} \sim GP\left(\mu_p(z),K_p(z,z')\right),
\end{equation}
where $\mu_p(z),K_p(z,z')$ are respectively, the posterior mean and covariance functions of the GP.
Given this posterior, we can probabilistically predict the decision maker's choice for any new finite choice set $B$, that is we can compute $P(C(B)=B'|\mathcal{D})$ for each $B' \subseteq B$.

When $d=1$ (scalar utility), the problem simplifies to the problem of learning a complete preference (aka standard preference learning).

Note that, for standard preference learning, the posterior is a SkewGP \cite{benavoli2020skew,benavoli2021preferential,benavoli2021}. It means that the posterior is asymmetric (skewed). However, for the analytical derivation of the results in the paper, we will approximate it with a GP (which has Gaussian marginals). We can use three methods to compute this approximation: (i) Laplace's approximation  \cite{mackay1996bayesian,williams1998bayesian}; (ii) Expectation Propagation \cite{minka2001family}; (iii) Kullback-Leibler divergence minimization \cite{opper2009variational}, including Variational approximation \cite{gibbs2000variational} as a particular case. 

\begin{example}
\label{ex:2}
To introduce the problem of choice-function learning, we make the simplifying assumption that the utility over the acts in Example \ref{ex:1} can be represented by a function $\nu(x)$ with the act $x$ be encoded as $x\in \mathcal{X}=[1,9]\subset \mathbb{R}$.\footnote{Here, we assume that $\mathcal{X}$ is infinite, while previously we assumed a finite $\mathcal{X}$. As mentioned earlier, we will not go into topological details in this paper. For example, in the current example it is natural to assume continuity of preferences and utility. Depending on the specific problem, topological constraints, as well as other conditions such as monotonicity, can be imposed.} This allows us to plot the utility in a Cartesian diagram.  We assume that in her mind, $\bH$'s taste as a function of $x$ is as depicted in Figure \ref{fig:butter}, where higher values indicate a stronger preference.  $\bR$ does not know $\bH$'s taste and learn it indirectly from $\bH$'s  preferences, such as
$$
\begin{aligned}
\mathcal{D}=\{&6.5 \succ 3.5, 7 \succ 5, 6.5 \succ 5.5, 3.5 \succ 8.5,\\
&1 \succ 9, 7 \succ 1.5, 4.5 \succ 7.5, 3.5 \succ 4\}
\end{aligned}
$$
where $6.5 \succ 3.5$ means that $\bH$ prefers the act $x=6.5$ over the act $y=3.5$. These preferences have been generated according to the utility function depicted in Figure \ref{fig:butter}.

\begin{figure}[h!]
\centering
\includegraphics[width=7cm]{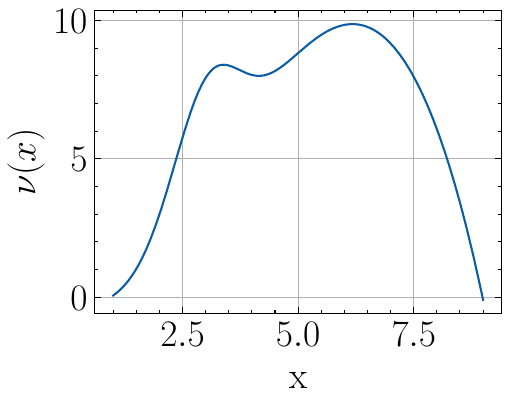}
\caption{Hidden utility in $\bH$'s mind.}
\label{fig:butter}
\end{figure}

Since there is only a single latent utility dimension ($d=1$), this problem reduces to learning a complete preference order. $\bR$ will do it by placing a GP prior over the unknown utility $\nu(x)$, as depicted in Figure \ref{fig:butterprior}. The likelihood is simply:
\begin{equation}
\label{eq:like}
p(\mathcal{D}|\nu)=I_{\{\nu(6.5)>\nu(3.5)\}}I_{\{\nu(7)>\nu(5)\}}\cdots I_{\{\nu(3.5)>\nu(4)\}}
\end{equation}
where $I_{\{\nu(6.5) > \nu(3.5)\}}$ denotes the indicator function, which is equal to 1 when the condition in the subscript holds true, and 0 otherwise. 
The posterior computed based on the $8$ preferences above is shown in Figure \ref{fig:butterposterior}.
The posterior computed based on a total of $30$ preferences  is shown in Figure \ref{fig:butterposterior30}. It is important to note that $\bR$ can never fully estimate the utility hidden in $\bH$'s mind because this utility is not identifiable from preferences alone. $\bR$ can only learn it up to a non-decreasing transformation.
This explains why the posterior mean in Figure \ref{fig:butterposterior30} is not `exactly' equal to the original utility in Figure \ref{fig:butterprior}, but they both generate the same preferences.

\begin{figure}[h!]
\centering
\includegraphics[width=9cm]{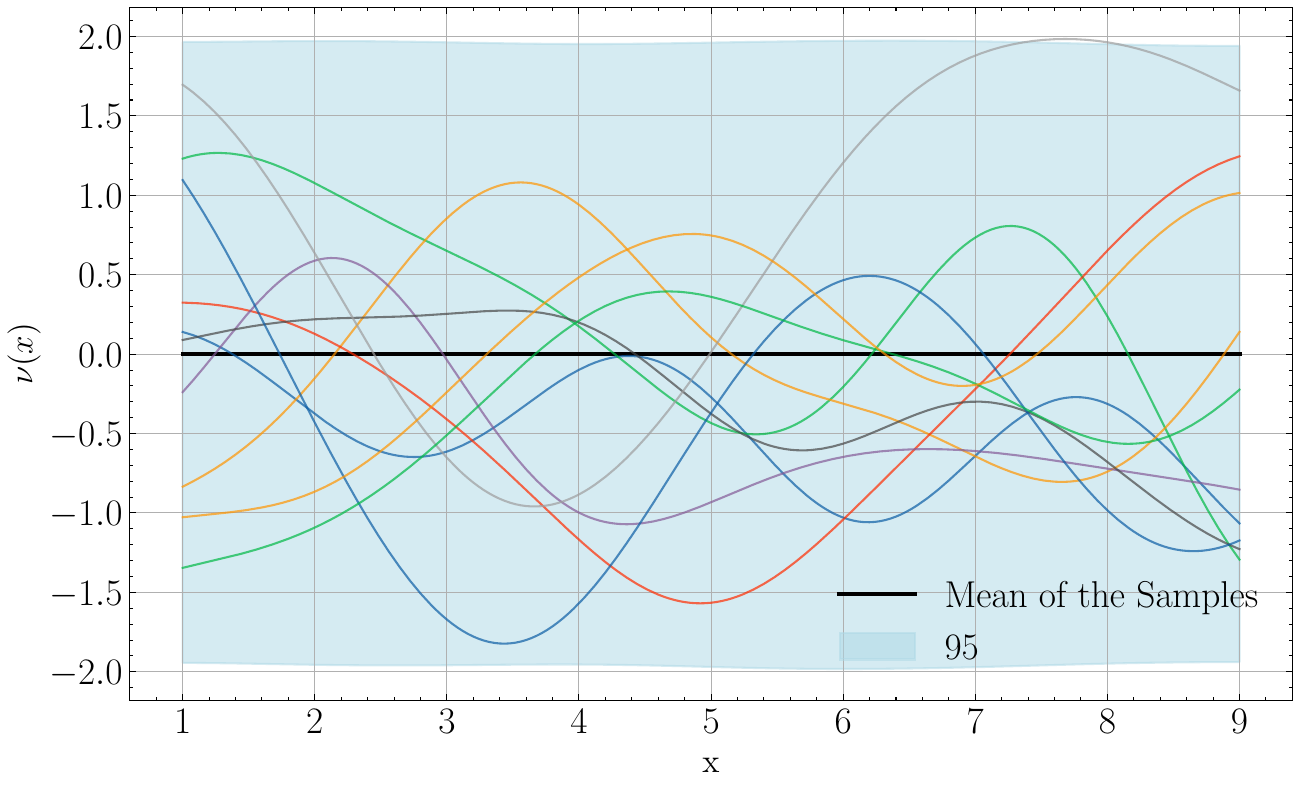}
\caption{GP prior: mean function (black line), 95\% credible region (blue shaded area), and 10 samples of $\nu(x)$, each shown in a different colour.
}
\label{fig:butterprior}
\end{figure}

\begin{figure}[h!]
\centering
\includegraphics[width=9cm]{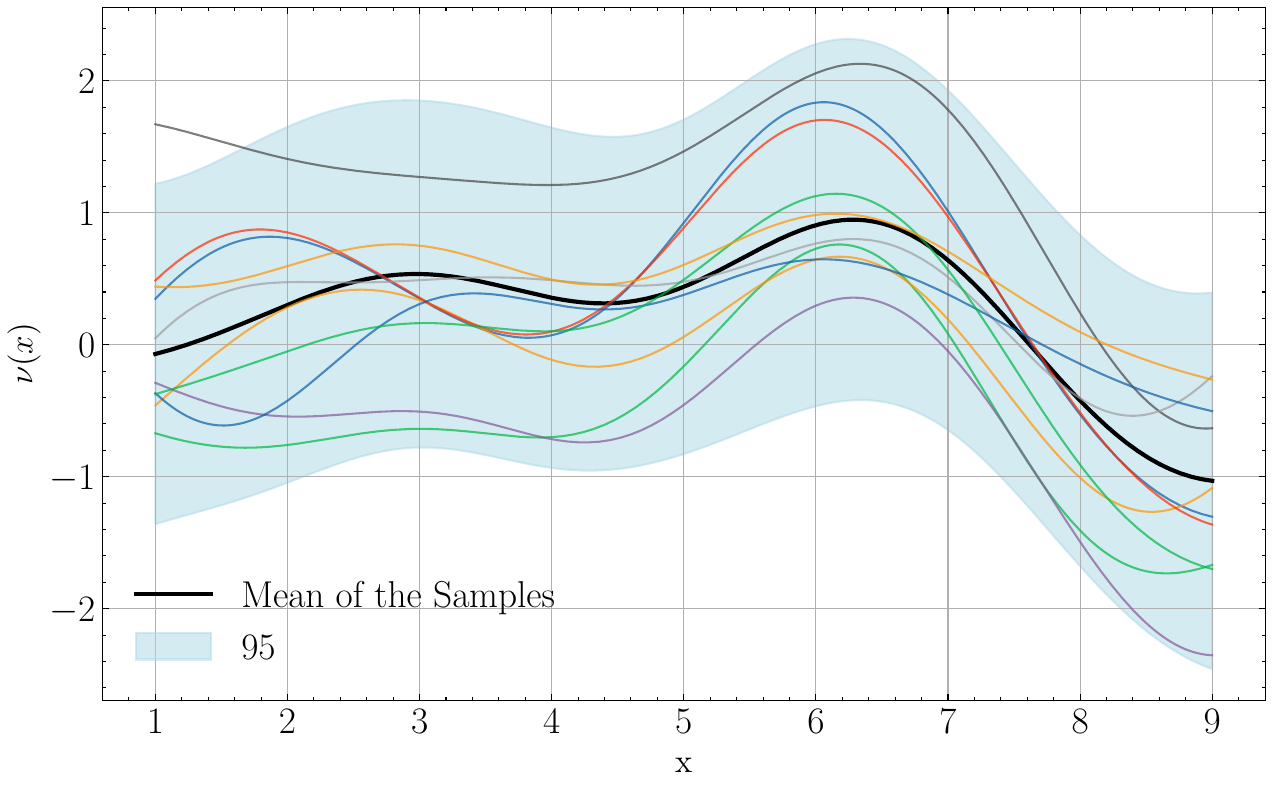}
\caption{GP posterior: mean function (black line), 95\% credible region (blue shaded area), and 10 samples of $\nu(x)$, each shown in a different colour.
}
\label{fig:butterposterior}
\end{figure}

\begin{figure}[h!]
\centering
\includegraphics[width=9cm]{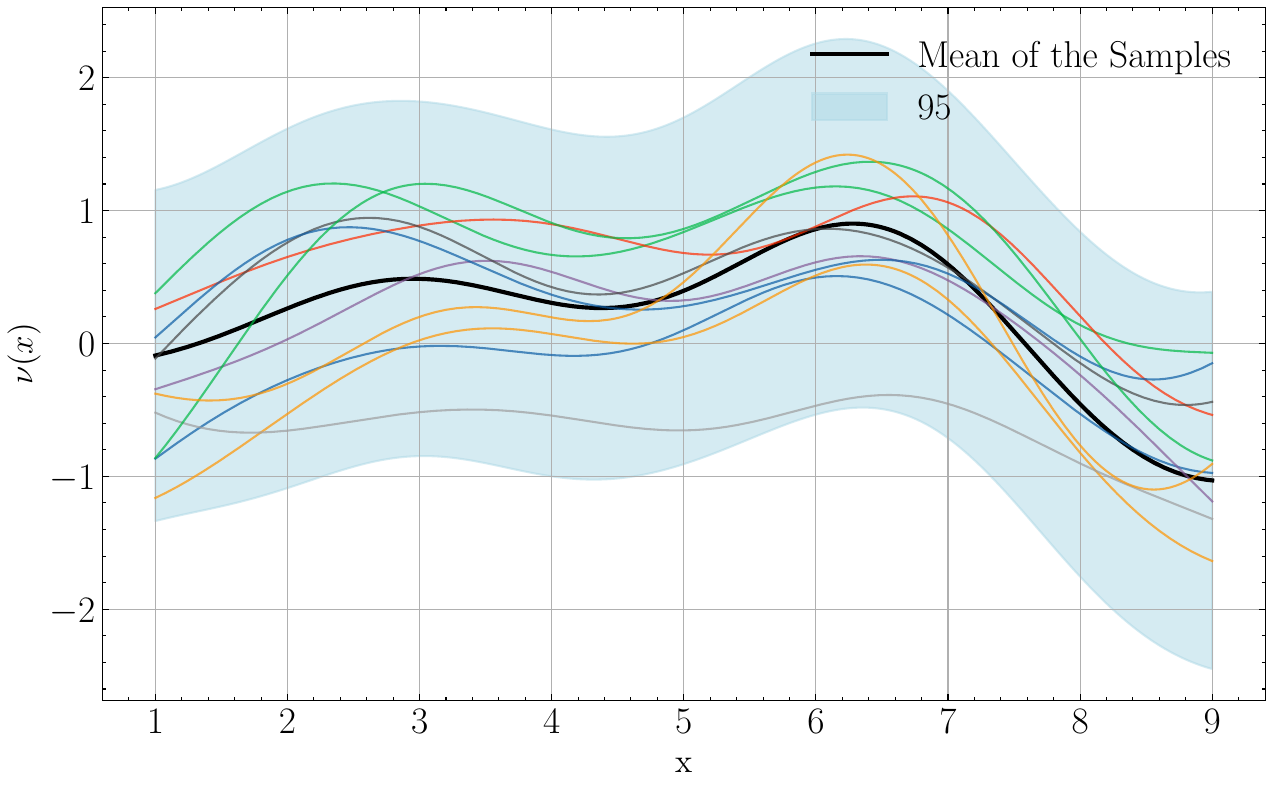}
\caption{GP posterior: mean function (black line), 95\% credible region (blue shaded area), and 10 samples of $\nu(x)$, each shown in a different colour.
}
\label{fig:butterposterior30}
\end{figure}
\end{example}

\begin{example}
\label{ex:3}
In the previous example, we have assumed a single utility $\nu(x)$. Here we instead assume that $\bH$ has in mind more than one utility, for instance $\nu_1(x),\nu_2(x)$ as depicted in Figure \ref{fig:choice2}-left.

\begin{figure}[h!]
\centering
\includegraphics[width=6.5cm]{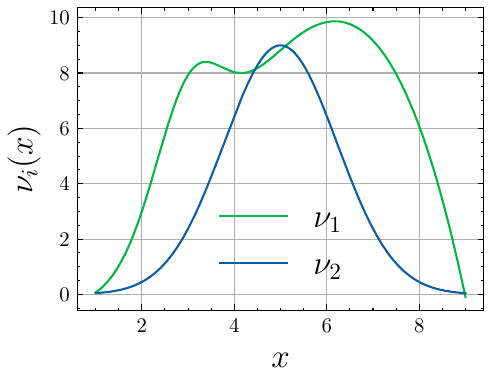}
\includegraphics[width=6.5cm]{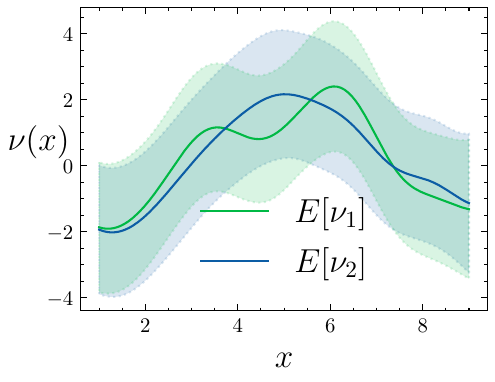}
\caption{Left: hidden utilities in $\bH$'s mind. Right: GP posterior marginals for the two utilities estimated from choice data, the mean functions are shown as lines and the shaded regions represent  the 95\% credible intervals.}
\label{fig:choice2}
\end{figure}

 $\bR$ does not know $\bH$'s utilities and learn them indirectly from $\bH$'s  choices. We assume that in her mind $\bH$ makes choices based on the mechanism \eqref{eq:pseudoratio}, leading to the choices
$$
\begin{aligned}
\mathcal{D}=\{&(\{2.8,6.8\},\{6.8\}), (\{5.2,6.8\},\{5.2,6.8\}),(\{1.2,7.7\},\{7.7\}),\dots\}
\end{aligned}
$$
where each element of $\mathcal{D}$ is a choice given a set of options, i.e., $(A_i,C(A_i))$. For instance, $(\{2.8,6.8\},\{6.8\})$ means that given the options $x_1=2.8,x_2=6.8$, $\bH$ selected $6.8$. This holds because the utilities $[\nu_1(6.8),\nu_2(6.8)]=[9.47,3.05]$
are both larger than $[\nu_1(2.8),\nu_2(2.8)]=[7.25,1.79]$. Conversely, in \\  $(\{5.2,6.8\},\{5.2,6.8\})$  we have a trade-off. Indeed, $[\nu_1(5.2),\nu_2(5.2)]=[9.1,8.88]$ dominates $6.8$ in the second utility, but it is dominated in the first utility.
To model incomparability, $\bR$ needs to learn two utilities and will do it by placing a GP prior over the unknown utilities $\nu_1(x),\nu_2(x)$. The likelihood is $ p(\mathcal{D}|\boldsymbol{\nu}(X))=\prod_{k=1}^m p(C(A_k),A_k|\boldsymbol{\nu}(X))$, where \cite{benavoli2024tutorial}:
\begin{equation}
  \label{eq:likelihoodpseudorat}
   \begin{aligned} 
   p(C(A_k),A_k|\boldsymbol{\nu}(X))&=\prod\limits_{\{{ o},{v}\} \in C(A_k)}\Bigg( 1-\prod_{i=1}^d I_{\{\nu_i({ o})-\nu_i({ v})>0\}}-\prod_{i=1}^d I_{\{\nu_i({ v})-\nu_i({ o})>0\}}\Bigg)\\
         &\prod_{{ v} \in R(A_k)}\Bigg(\prod_{i=1}^d  \left(1-\prod_{{ o} \in C(A_k)} I_{\{\nu_i({ v})-\nu_i(o)>0\}}\right)\Bigg),\\
   \end{aligned}    
 \end{equation}
 with $d=2$ (number of utilities). The product in the first line imposes the condition that each object in $C(A)$ is not dominated in both the utilities (according to the condition \eqref{eq:pseudoratio}).
 The product in the second line  means that for all ${v} \in R(A_k)=A_k\backslash C(A_k)$, it is not true that the value of any latent utility in ${v}$ is higher than their value in ${o} \in C(A)$. 
 
The posterior computed based on a dataset $\mathcal{D}$ including $100$ choices among $60$ elements sampled uniformly from $\mathcal{X}$  is shown in Figure \ref{fig:choice2}. It can be noticed how the posterior is able to learn back the utilities from the choice dataset $\mathcal{D}$. It learns the utilities that belongs to an equivalence class of utilities, which are related by a monotonic increasing transformation and lead to the same choice dataset.\footnote{We used the Python library \texttt{PrefGP} (\url{https://github.com/benavoli/prefGP}) to compute the posterior in this and previous example.}
\end{example}

\subsection{Learning from choices under bounded-rationality}
\label{sec:rum}
In standard decision theory, it is assumed that the decision maker is rational. However, due to cognitive limitations, we can generally only assume that the decision maker is bounded-rational \cite{simon1990bounded}. A typical instance arises when the decision maker has limited time to make a decision (such as when playing chess), and, under time pressure, may resort to a random choice. 
Other cases involve limitations in computational resources \cite{gershman2015computational,benavoli2019sum} and discernibility \cite{luce1956semiorders}. In the latter, the decision-maker may struggle to differentiate between two acts with similar utilities, resulting in a random choice. These situations are usually modelled through \textit{random utility models} \cite{train2009discrete}, where noise with a certain distribution is included in the choice mechanism. For instance, in this case, the decision maker is assumed to choose $C(\{y,z\})=\{z\}$ when
\begin{equation}
\label{eq:unoise}
\nu(z)+n(z)>\nu(y)+n(y)
\end{equation}
where $n(z),n(y)$ are independent noises. If $n(z),n(y) \sim N(0,\sigma^2)$ \cite{Thu27}, then
\begin{equation}
\label{eq:probit}
\begin{aligned}
P(C(\{z,y\})=\{z\})&=\Phi\left(\frac{\nu(z)-\nu(y)}{\sqrt{2}\sigma}\right),\\
P(C(\{z,y\})=\{y\})&=\Phi\left(\frac{\nu(y)-\nu(z)}{\sqrt{2}\sigma}\right)=1-\Phi\left(\frac{\nu(z)-\nu(y)}{\sqrt{2}\sigma}\right),\\
\end{aligned}
\end{equation}
where $\Phi$ is the CDF of the standard normal distribution. In the paper, we use $\phi$ to denote the PDF of the standard normal distribution. It can be noticed that when $\nu(z) - \nu(y) = 0$, the decision maker chooses between the two options with probability $1/2$. However, when the difference $|\nu(z) - \nu(y)|$ is very large compared to $\sigma$, the decision maker  will choose with almost certainty. Therefore, this random utility model captures a limit of discernibility through the discernibility parameter $\sigma$ \cite{benavoli2024tutorial}. Note that, we can use the GP model to learn from this noisy data, in this case the likelihood \eqref{eq:like} becomes
\begin{equation}
\label{eq:likeprobit}
p(\mathcal{D}|\nu)=\Phi\left(\tfrac{\nu(6.5)-\nu(3.5)}{\sqrt{2}\sigma}\right)\cdots \Phi\left(\tfrac{\nu(3.5)-\nu(4)}{\sqrt{2}\sigma}\right).
\end{equation}
This is a possible model for bounded-rationality. Another alternative model can be defined if we allow the decision-maker to state that two alternatives are incomparable when their utility difference is below the limit of discernability, then we can represent the decision-maker choices through this  choice function model \cite{luce1956semiorders}:
\begin{equation}
\label{eq:scalarubounded}
\begin{aligned}
 C(\{z,y\})=\left\{\begin{array}{ll}
 \{z\} & \text{if } \nu(z) > \nu(y) + \sigma,\\
  \{y\} & \text{if } \nu(y) > \nu(z) + \sigma,\\
    \{z,y\} & \text{if } |\nu(z)-\nu(y)| \leq  \sigma,
 \end{array}\right.
 \end{aligned}
\end{equation}
where $\sigma > 0$ is the limit of discernibility. In this setting, given the choice set $\{y, z\}$, the decision-maker may refuse to select either alternative if, due to her limit of discernibility, she is unable to distinguish between them. This gives rise to an interval of \textit{indifference},\footnote{This is not an indifference relation, it is an incomparability due the limited discernability of the decision-maker.} in the sense that some pairs of alternatives are indistinguishable to her.

\begin{remark}
By comparing \eqref{eq:probit} and \eqref{eq:scalarubounded}, as depicted in Figure \ref{fig:incompcomp}, we observe that in situations of limit of discernability, forcing the decision-maker to express a preference leads to errors (i.e., noisy preference data). Specifically, there is a non-zero probability of stating an incorrect preference, as defined by \eqref{eq:probit} and illustrated in blue in Figure \ref{fig:incompcomp}. Therefore, \eqref{eq:probit} can be interpreted as a model capturing the noise affecting preference data when the decision-maker is forced to choose between alternatives $y$ and $z$, despite being unable to discriminate between them due to her limited discernibility. 
Both these models represent forms of bounded-rationality. The choice function in \eqref{eq:scalarubounded} is distinct from the incomparability discussed in Example \ref{ex:3}. In Example \ref{ex:3}, the inability of the decision maker to select a single act arises from the presence of multiple utility functions and, therefore, it is a proper instance of incomparability.
\end{remark}
\begin{figure}[h!]
\centering
\includegraphics[width=8cm]{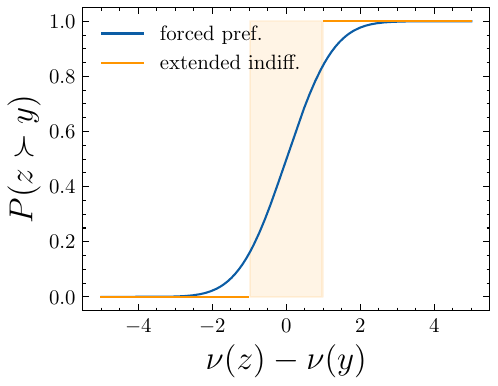}
\caption{Comparison between the likelihood in \eqref{eq:probit} (in blue) and the choice function model \eqref{eq:scalarubounded} (in orange), both computed with $\sigma=1$. The y-axis reports the probability that $z \succ y$ as a function of the difference of utilities. The orange shaded area is the region of extended indifference for the model \eqref{eq:scalarubounded}.}
\label{fig:incompcomp}
\end{figure}
In this section, we provided examples of bounded rationality for the scalar optimisation choice mechanism in \eqref{eq:scalaru}. Similar models can also be introduced for the vector mechanisms. For instance, we can obtain a bounded-rationality model for the likelihood in \eqref{eq:likelihoodpseudorat} by replacing the indicators $I_{\{\nu_i({ o})-\nu_i({ v})>0\}}$ with $\Phi(\tfrac{\nu_i(o)-\nu_i(v)}{\sqrt{2}\sigma})$ \cite{benavoli2024tutorial}.

\subsection{Modified signalling games}
\label{sec:signalling}
\emph{Signalling games} \cite{spence1974market,gibbons1992game} specifically refer to a class of two-player games with incomplete information, where one player possesses private information (the Sender), while the other does not (the Receiver). The timing and payoffs of the game are  as
follows:
\begin{enumerate}
    \item Nature draws a state of the world (a type) $t_i \in T$  according to a probability distribution $p(t)$. 
        \item The Sender observes $t_i$ and then chooses a message $m_j \in M$ from a finite set of messages $M$ and sends it to the Receiver.
    \item The Receiver observes $m_j$ (but not $t_i$) and then chooses an action $a_k \in A$ from a finite set of possible actions $A$.
    \item If $a_k \in A' \subset A$  then the Sender chooses an action $b_l \in B$ from a finite set of possible actions $B$. Otherwise, $b_l= \varnothing$ (null decision).
    \item Payoffs for Sender and Receiver are given by $u_S(t_i, m_j, a_k,b_l)$
    and, respectively,  $u_R(t_i, m_j, a_k,b_l)$.
\end{enumerate}
Step 4 in the game is typically absent in standard signalling games, which is why we refer to it as a \textit{modified} signalling game. Here, $A'$ denotes the subset of actions for which the Sender can make an additional move.

\begin{crequirement}{1}
\label{req:1}
After observing any message $m_j$ from $M$, the
Receiver must have a belief about which types could be the right one. Denote
this belief by the probability $p(t|m_j)$.
\end{crequirement}
\begin{crequirement}{2R}
\label{req:2R}
For each $m_j \in M$, the Receiver's action
$a^*(m_j)$ must maximise the Receiver's expected utility, given the belief $p(t|m_j)$ about which types could have sent $m_j$. That is,
\begin{equation}
a^*(m_j) \in \arg\max_{a_k \in A} \int_{T} u_R(t,m_j,a_k,b^*(a_k))dp(t|m_j).
\end{equation}
\end{crequirement}
Note that $b^*(a_k)$ is not known by the Receiver, but it depends on $t$ and, therefore, the Receiver can compute this expectation.
Requirement \ref{req:2R} also applies to the Sender, but the Sender has complete information (and hence a trivial belief), so Requirement  \ref{req:2S}  is simply that the
Sender's strategy be optimal given the Receiver's strategy:
\begin{crequirement}{2S}
\label{req:2S}
For each $t_i \in T$, the Sender's message
$m^*(t_i)$ and the action $b^*(a^*(m_j))$ must maximize the Sender's utility, given the Receiver's strategy $a^*(m_j)$. That is, we have that:
\begin{equation}
\begin{aligned}
&\big(  m^*(t_i),b^*(a^*(m_j)) \big)\\
&\in \arg\max_{m \in M, b \in B} u_s(t_i,m,a^*(m),b(a^*(m))).
\end{aligned}
\end{equation}
\end{crequirement}
Finally, given the Sender's strategy $m^*(t_i)$, let $T_j$ denote the set
of types that send the message $m_j$. That is, $t_i$ is a member of
the set $T_j$ if $m^*(t_i) = m_j$. If $T_j$ is nonempty then the information
set corresponding to the message $m_j$ is on the equilibrium path; otherwise, $m_j$ is not sent by any type and so the corresponding
information set is off the equilibrium path. For messages on the
equilibrium path, we have that
\begin{crequirement}{3}
\label{req:3}
For each $m_j \in M$, if there exists $t \in T$ such that  $m^*(t) = m_j$, then the Receiver's belief at the information set corresponding to $m_j$ must follow from Bayes' rule and the Sender's strategy:
\begin{equation}
    p(t|m_j)%
    =\frac{I_{\{m^*(t) = m_j\}}(t)p(t)}{\int_{T}I_{\{m^*(t) = m_j\}}(t)dp(t)}.
\end{equation}
\end{crequirement}

\begin{definition}
A pure-strategy perfect Bayesian equilibrium in a signalling game is a triplet of strategies $(m^*(t_i),a^*(m_j),b(a^*(m_j))$  and a belief $p(t|m_j)$
satisfying Requirements \ref{req:1}--\ref{req:3}.
\end{definition}

\section{The AI-assistance game}
\label{sec:contrib1}
In this section, we revisit the AI-assistance game introduced in \cite{hadfield2017off} through the lens of signalling games, as depicted in Figure \ref{fig:AIassistancegame}. In this game, the human ($\bH$), acts as the Sender and the robot ($\bR$), as the Receiver, with $T$ representing the set of possible types for $\bH$. Each type $t_i \in T$ characterises $\bH$'s ``taste'' and degree of bounded rationality for an underlying decision problem $(\mathcal{W}, \mathcal{O}, \mathcal{X}, C)$. 
\begin{figure}[h]
\centering
\includegraphics[width=14cm]{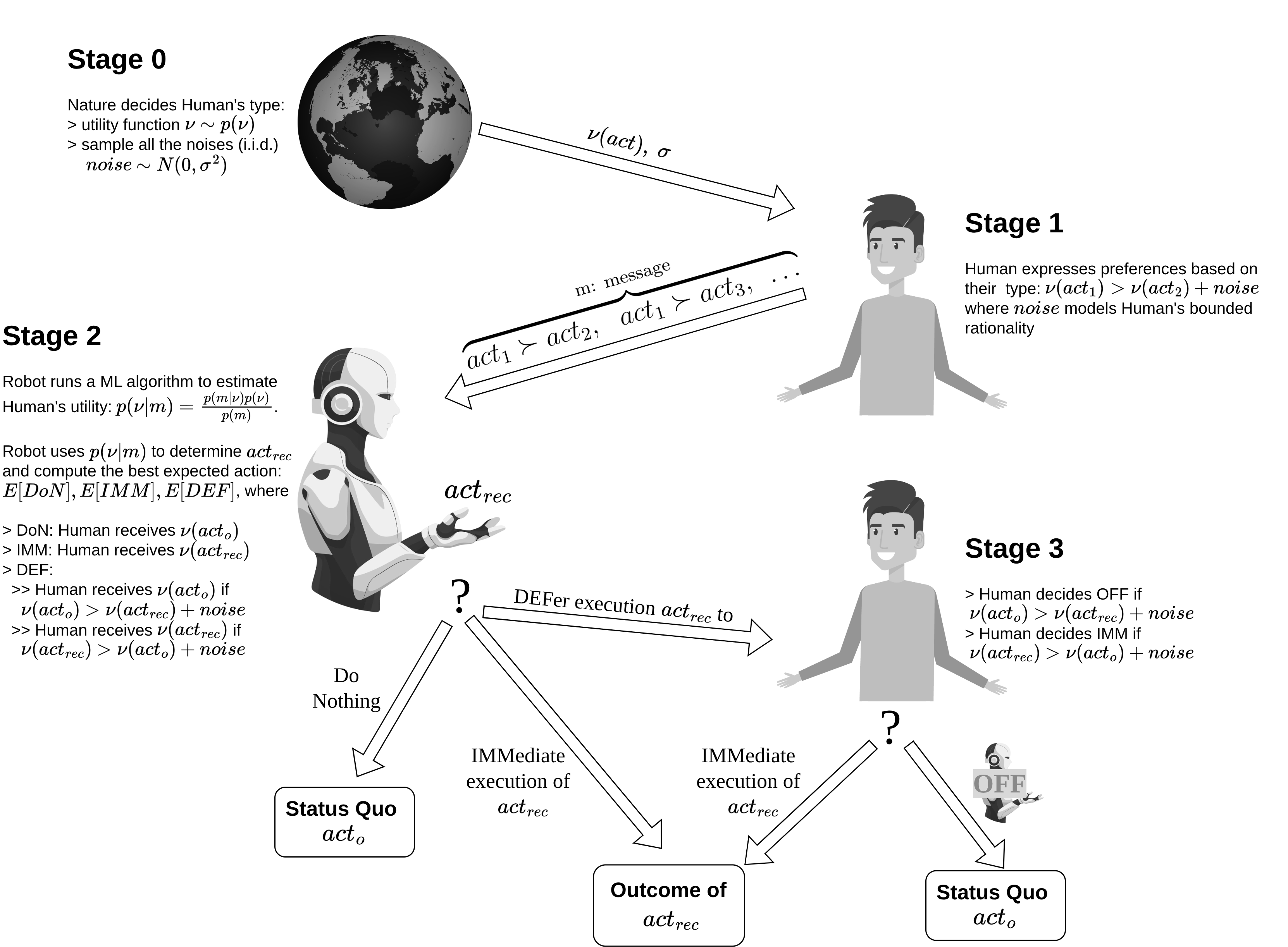}
\caption{AI assistance game as a signalling game}
\label{fig:AIassistancegame}
\end{figure}
As studied in \cite{hadfield2017off}, we initially consider a bounded-rationality model similar to \eqref{eq:unoise}, where a single scalar utility is assumed (that is, we assume completeness). Thus, the type $t_i$ is a realisation of a utility function $\nu$ and (a vector of independent) noises\footnote{It is a vector because the noise is also present in the message, that is in the choice dataset.} ${ \bf n}$, where $p({ \bf n}) = N({ \bf n}; 0, I\sigma^2)$ and $p(\nu) = GP(\nu; \mu_0, K_0)$, for some mean function $\mu_0$, kernel function $K_0$, and parameter $\sigma^2$ ($I$ is the identity matrix). We further assume that all these parameters are common knowledge in the game.

The message $m_j \in M$ encapsulates $\bH$'s preferences/choices over acts in $\mathcal{X}$, relative to the underlying decision problem $(\mathcal{W}, \mathcal{O}, \mathcal{X}, C)$. That is, $m_j$ represents a dataset of choices as in \eqref{eq:choiciedata}. Since we are considering a single utility function, we assume a choice mechanism as in \eqref{eq:scalaru}, leading to binary preferences.
We first assume that the payoffs do not directly depend on the message $m_j$. In signalling games, this assumption is referred to as \textit{cheap-talk} \cite{gibbons1992game}. Furthermore, we assume that each type $t_i$ can only choose one message, i.e., $\bH$ can only send the message determined by Nature through $\bH$'s type. This is the setting studied in \cite{hadfield2017off}. The message is a choice dataset, such as the one in \eqref{eq:choiciedata}, which $\bR$ uses to estimate $\bH$'s utility using the approach described in Section \ref{sec:learning} and illustrated in Example \ref{ex:2}.

\begin{tcolorbox}[width=\linewidth, sharp corners=all, colback=white!95!black]
The key difference between  an AI-assistance game and a general signalling game is that in the former $\bR$ aims to maximise $\bH$'s utility. This is why we refer to $\bR$ as being an assistant.
\end{tcolorbox}

We consider two acts $x,o \in \mathcal{X}$ relative to the underling decision problem $(\mathcal{W}, \mathcal{O}, \mathcal{X}, C)$. We assume that $o$ is the status-quo and $x$ is some new act that $\bR$ recommends.

\begin{example}
\label{ex:31} In our customer service example, we are considering a scenario, where the  AI-assistant  $\bR$ is helping $\bH$ to make a decision. In this setting, $o$ represents the act chosen by $\bH$ (e.g., ``refund immediately''). Instead, $x$ is a new  act recommended by  $\bR$ to help $\bH$ to achieve a better decision (to maximise $\bH$'s utility).
 \end{example}

We will address the optimal selection of $x$ in Section \ref{sec:optimalsel}. In this section, we simply assume that $\bR$ will choose an act $x$ that is preferable to $o$ whenever possible. 

\begin{tcolorbox}[width=\linewidth, sharp corners=all, colback=white!95!black]
This explains why $\bH$ is willing to engage in the game with $\bR$: $\bH$ expects $\bR$ to recommend acts that improve upon the status quo, since $\bH$ does not know how to identify such acts on her own.
\end{tcolorbox}

In this context, the available actions for $\bR$ are $A = \{IMM, DEF, DoN\}$. 
$IMM$ means `immediate decision', that is $\bR$ will implement $x$. DoN means `do nothing'. $DEF$ means $\bR$ will defer to $\bH$. In the signalling game, this implies that $A' = \{DEF\}$. In the case $\bR$ defers to $\bH$, then $\bH$ has two possible actions, $B = \{OFF, \neg OFF\}$, either to switch off $\bR$ or to allow $\bR$ to implement $x$. Therefore, the payoff for $\bR$ is
\begin{equation}
\label{eq:payoffprop1R}
\begin{aligned}
u_R(t_i,m_j,IMM,\varnothing)&=\nu(x),\\
u_R(t_i,m_j,DoN,\varnothing)&=\nu(o),\\
u_R(t_i,m_j,DEF,b^*(DEF))&=\nu(o)I_{\{\nu(o)+n(o)>\nu(x)+n(x)\}}
\\&+\nu(x)I_{\{\nu(x)+n(x)>\nu(o)+n(o)\}},
\end{aligned}
\end{equation}
where we have denoted the type and message as $t_i,m_j$. Indeed, if $\bR$ chooses $IMM$, then the payoff for  $\bR$ is equal to the utility that $\bH$ receives from the act $x$. If $\bR$ chooses $DoN$, then the payoff  for $\bR$ is equal to the utility of the status-quo for $\bH$  (utility relative to the act $o$).
The last case represents the payoff for the action $DEF$, where the action depends on $\bH$'s move in the game, which is encapsulated by the indicator functions. Given $\bH$ is a bounded rational agent, $\bH$ chooses the action $OFF$ or $\neg OFF$ based on $\bH$'s noisy utility. If $\bH$ chooses $OFF$, then $\bH$  receives the utility of the status-quo $o$, otherwise the utility of $x$.

$\bR$ does not know $\bH$'s utility $\nu$. Therefore, according to \textbf{Requirement 2R}, $\bR$ will choose the action that maximises the expected value of the payoff in \eqref{eq:payoffprop1R}, where the expectation is computed with respect to $p(t | m_j)$, i.e., the posterior distribution over $\bH$'s type (utility $\nu$ for the underlying decision problem $(\mathcal{W}, \mathcal{O}, \mathcal{X}, C)$) learned by $\bR$ from the message $m_j$ (which is a dataset of preferences). We can then prove the following lemma.

\begin{lemma}
\label{lem:1}
Assume that $p(\nu|\mathcal{D})=GP\left(\nu;\mu_p
,K_p\right)$ is the GP posterior computed by $\bR$ from the prior   $p(\nu)=GP\left(\nu;\mu_0,K_0\right)$, the bounded-rationality likelihood \eqref{eq:probit} and the message $m_j=\mathcal{D}$, then
 the expected payoffs of $\bR$'s actions  are:
 {\small
\begin{align}
\nonumber
&DEF:  \int_{T} \Big(\nu(o)I_{\{\nu(o)+n(o)>\nu(x)+n(x)\}}+\nu(x)I_{\{\nu(o)+n(o)<\nu(x)+n(x)\}} \Big)\\
\nonumber
&dp(\nu(x),n(x),\nu(o),n(o)|m_j)\\
\nonumber
&=\mu_p(x)\left(1-\Phi\left(\tfrac{\mu_p(o)-\mu_p(x)}{\sqrt{K_p(x,x)+2\sigma^2+K_p(o,o)-2K_p(x,o)}}\right)\right)\\
\nonumber
& + \tfrac{K_p(x,x)-K_p(x,o)}{\sqrt {K_p(x,x)+2\sigma^2+K_p(o,o)-2K_p(x,o)}}\phi\left(\tfrac{\mu_p(o)-\mu_p(x)}{\sqrt{K_p(x,x)+2\sigma^2+K_p(o,o)-2K_p(x,o)}}\right)\\
\nonumber
&+\mu_p(o)\left(1-\Phi\left(\tfrac{\mu_p(x)-\mu_p(o)}{\sqrt{K_p(x,x)+2\sigma^2+K_p(o,o)-2K_p(x,o)}}\right)\right)\\
\label{eq:DEFexp}
& + \tfrac{K_p(o,o)-K_p(x,o)}{\sqrt {K_p(x,x)+2\sigma^2+K_p(o,o)-2K_p(x,o)}}\phi\left(\tfrac{\mu_p(x)-\mu_p(o)}{\sqrt{K_p(x,x)+2\sigma^2+K_p(o,o)-2K_p(x,o)}}\right),\\
\label{eq:IMMexp}
&IMM:  \int_{T}\nu(x)dp(\nu(x),n(x),\nu(o),n(o)|m_j)=\mu_p(x),\\
\label{eq:OFFexp}
&DoN:  \int_{T}\nu(o)dp(\nu(x),n(x),\nu(o),n(o)|m_j)=\mu_p(o),
\end{align}}
where  {\small $$p(\nu(x),n(x),\nu(o),n(o)|m_j)
=p(\nu(x),\nu(o)|m_j)
p(n(x),n(o))$$}
with {\small $p(n(x),n(o))=N(n(x);0,\sigma^2)N(n(o);0,\sigma^2)$} and
{\small
\begin{align}
\label{eq:multiprior3}
&p(\nu(x),\nu(o)|m_j)= N\left(\begin{bmatrix}
\nu(x) \\ 
\nu(o) 
\end{bmatrix};\begin{bmatrix}
\mu_p(x) \\ 
\mu_p(o) 
\end{bmatrix},\begin{bmatrix}
K_p(x,x) & \hspace{-2mm} K_p(x,o) \\ 
K_p(o,x) & \hspace{-2mm}  K_p(o,o) 
\end{bmatrix}\right).
\end{align}}
\end{lemma}
Proofs of this and the following results are provided in Appendix. In Lemma \ref{lem:1}, we have exploited the fact that the marginal of a GP is a multivariate normal (equation \eqref{eq:multiprior3}). This allows us to compute the expectations of the utilities for the three actions, DoN, IMM and DEF for $\bR$.
We then introduce the following definitions:
\begin{definition}
\label{def:cases}
We say that:
\begin{itemize}
\item $S$ is \textbf{rational} whenever $\sigma\rightarrow 0$;  $S$ is \textbf{bounded-rational} otherwise;
\item $R$ has \textbf{no uncertainty} on $S$'s utility whenever \\ $K_o(x,x),K_o(o,o),K_o(o,x)\rightarrow 0$ (that is the prior becomes a Dirac's delta); otherwise $R$ has \textbf{uncertainty}.
\end{itemize}
\end{definition}
In the signalling game defined in Section \ref{sec:signalling}, note that $R$ having no uncertainty about $S$'s utility implies that $R$ possesses perfect knowledge of $S$'s utility. This is because the prior $p(t)$ is  common knowledge within the game.
We can then prove the following result. 
\begin{proposition}
\label{prop:1}
The optimal decisions for $\bR$ are:
\begin{itemize}
\item If $S$ is \textbf{rational} and $R$ has \textbf{no uncertainty}, then $DEF$ is  never dominated by $IMM,DoN$.\footnote{The $DEF$ action has the same payoff for $\bR$ as the best action between $IMM,DoN$.}
\item If $S$ is \textbf{bounded-rational} and $R$ has \textbf{no uncertainty}, then DEF is never optimal;
\item If $S$ is \textbf{rational} and $R$ has \textbf{uncertainty}, then DEF is always optimal.
\item If $S$ is \textbf{bounded-rational} and $R$ has \textbf{uncertainty}, then DEF is optimal if 
\eqref{eq:DEFexp} is larger or equal than the maximum between \eqref{eq:IMMexp} and \eqref{eq:OFFexp}.
\end{itemize}
\end{proposition}

Therefore, by interpreting the non-optimality of $DEF$ as the robot avoiding human supervision, we conclude that:

\begin{tcolorbox}[width=\linewidth, sharp corners=all, colback=white!95!black]
If $\bH$ is rational, then $\bR$  will never avoid human supervision. If $\bH$ is bounded-rational, a necessary condition for $\bR$ not to avoid human supervision is the presence of uncertainty.
\end{tcolorbox}

This result aligns with what was proven in \cite{hadfield2017off}; however, in our case, the statements have been rigorously established within the framework of signalling games, also incorporating the posterior uncertainty (in equation \eqref{eq:multiprior3}) derived from preference learning. This means that the result was proven under  the principles  P1--P3 and using a real machine learning model to learn from preferences, as we will demonstrate in Section \ref{sec:numeric1} by verifying the statements in Proposition \ref{prop:1} through numerical experiments.

It is particularly interesting to note that the conclusions of Proposition \ref{prop:1} would still hold, even if the absence of uncertainty  did not imply perfect knowledge of $\nu$ for the robot.\footnote{The result of Proposition \ref{prop:1} depends only on whether the posterior is deterministic or not.}
For this reason, the conclusion of this result is that (see also \cite{hadfield2017off,russell2019human}):
\begin{tcolorbox}[width=\linewidth, sharp corners=all, colback=white!95!black]
We shall not build $\bR$ to estimate $\nu$ through a model, such as a \textit{neural network}, that only provides deterministic point predictions (equivalently, a model that does not provide any probabilistic measure of uncertainty).
\end{tcolorbox}
Indeed, the \textit{necessity of uncertainty}  in Proposition \ref{prop:1} strongly supports the use of probabilistic methods in AI. We will revisit this in Section \ref{sec:numeric1}.

\subsection{Selection strategies for the robot}
\label{sec:optimalsel}
 In machine learning, the optimal selection of $x$ can be done  through a technique called Bayesian optimisation  \cite{shahriari2015taking} and, in particular, preferential Bayesian optimisation \cite{shahriari2015taking,gonzalez2017preferential,benavoli2021preferential, benavoli2023d,takeno2023towards,lin2022preference,Adachi_etal_2024_LoopExplain,astudillo2023qeubo}. We assume that $\bR$ selects the suggested act $x$ by solving the following optimisation problem:
 \begin{equation}
 \label{eq:prefBO}
 x=\arg\max_{x' \in \mathcal{X}} E[\text{acq}(x',o)],
 \end{equation}
 where the expectation is taken with respect to all unknowns\\ ($\nu(x),\nu(o),n(x),n(o)$) and $\text{acq}(x,o)$ is the so called \textit{acquisition function}, which is simply a function representing a certain selection strategy.\footnote{The optimisation problem may have multiple equivalent solutions and, in this case, we just take one of them.} We consider, for instance, the following three different acquisition functions $\text{acq}(x,o)$:

 \begin{align}
  &\text{natural:} && \max(\nu(x),\nu(o))=\nu(o)I_{\{\nu(o)>\nu(x)\}}+\nu(x)I_{\{\nu(x)>\nu(o)\}},\\
  &  \text{corporate:} && I_{\{\nu(x)+n(x)>\nu(o)+n(o)\}},\\
   &   \text{collaborative:} && \nu(o)I_{\{\nu(o)+n(o)>\nu(x)+n(x)\}}+\nu(x)I_{\{\nu(x)+n(x)>\nu(o)+n(o)\}}.
 \end{align}
For the first strategy, by solving \eqref{eq:prefBO}, $\bR$ chooses an act $x$ that maximally improves upon the status quo. The second strategy assumes that $\bR$ proposes an act that $\bH$ would select in the DEF case. We call this strategy \textit{corporate}, since in a repeated interaction between $\bR$ and $\bH$ it will lead to many messages between them.\footnote{For instance, most companies now offer forms of pay-per-use subscriptions to LLMs, so they profit more when users interact with the system frequently.} For the third strategy, $\bR$ suggests optimal acts that $\bH$ can recognise as optimal, given her bounded rationality. 
The first two acquisition functions have ben used in (preferential) Bayesian optimisation -- they are known as  expected-improvement and probability of improvement.
For the three acquisition functions, we can compute their expected value with respect the unknowns.

 \begin{lemma}
 \label{new:lemma}
 Consider $x,o \in \mathcal{X}$ and assume $\bR$ has the following posterior:
\begin{equation}
\label{eq:multipriorlem}
\begin{bmatrix}
\nu(x) \\
\nu(o)
\end{bmatrix} \sim N\left(\begin{bmatrix}
\mu_p(x) \\
\mu_p(o)
\end{bmatrix},\begin{bmatrix}
K_p(x,x) & K_p(x,o) \\
K_p(o,x) & K_p(o,o)
\end{bmatrix}\right),
\end{equation}
and $n(x),n(o) \sim N(0,\sigma^2)$ (independent noise). Then we have that
 \begin{align}
       \nonumber
 &E\left[\nu(o)I_{\{\nu(o)>\nu(x)\}}+\nu(x)I_{\{\nu(x)>\nu(o)\}}\right] =\mu_p(x)\Phi\left(\tfrac{(\mu_p(x)-\mu_p(o))}{\sqrt{K_p(x,x)+K_p(o,o)-2K_p(x,o)}}\right)\\
       \nonumber
& + \tfrac{K_p(x,x)-K_p(x,o)}{\sqrt {K_p(x,x)+K_p(o,o)-2K_p(x,o)}} \cdot \phi\left(\tfrac{\mu_p(o)-\mu_p(x)}{\sqrt{K_p(x,x)+K_p(o,o)-2K_p(x,o)}}\right)\\
\nonumber
&+\mu_p(o)\Phi\left(\tfrac{(\mu_p(o)-\mu_p(x))}{\sqrt{K_p(x,x)+K_p(o,o)-2K_p(x,o)}}\right)\\
&+ \tfrac{K_p(o,o)-K_p(x,o)}{\sqrt {K_p(x,x)+K_p(o,o)-2K_p(x,o)}} \cdot \phi\left(\tfrac{\mu_p(x)-\mu_p(o)}{\sqrt{K_p(x,x)+K_p(o,o)-2K_p(x,o)}}\right) 
\end{align}
\begin{align}
     & E\left[I_{\{\nu(x)+n(x)>\nu(o)+n(o)\}}\right] = \Phi\left(\tfrac{\mu_p(x)-
\mu_p(o)}{\sqrt{2\sigma^2+K_p(x,x)+K_p(o,o)-2K_p(x,o)}}\right)
\end{align}
\begin{align}
  \nonumber
   &   E\big[\nu(o)I_{\{\nu(o)+n(o)>\nu(x)+n(x)\}} +\nu(x)I_{\{\nu(x)+n(x)>\nu(o)+n(o)\}}\big] \\
     \nonumber
   &=  \mu_p(x)\Phi\left(\tfrac{(\mu_p(x)-\mu_p(o))}{\sqrt{K_p(x,x)+2\sigma^2+K_p(o,o)-2K_p(x,o)}}\right)\\
       \nonumber
& + \tfrac{K_p(x,x)-K_p(x,o)}{\sqrt {K_p(x,x)+2\sigma^2+K_p(o,o)-2K_p(x,o)}} \cdot \phi\left(\tfrac{\mu_p(o)-\mu_p(x)}{\sqrt{K_p(x,x)+2\sigma^2+K_p(o,o)-2K_p(x,o)}}\right)\\
\nonumber
&+ \mu_p(o)\Phi\left(\tfrac{(\mu_p(o)-\mu_p(x))}{\sqrt{K_p(x,x)+2\sigma^2+K_p(o,o)-2K_p(x,o)}}\right)\\
& + \tfrac{K_p(o,o)-K_p(x,o)}{\sqrt {K_p(x,x)+2\sigma^2+K_p(o,o)-2K_p(x,o)}} \cdot \phi\left(\tfrac{\mu_p(x)-\mu_p(o)}{\sqrt{K_p(x,x)+2\sigma^2+K_p(o,o)-2K_p(x,o)}}\right).
  \end{align}
 \end{lemma}
Hereafter, we assume that $\bR$ is always able to find the global maximum of the optimisation problem \eqref{eq:prefBO}. This assumption is, of course, unrealistic in practice, since 
$\bR$ is also a computationally bounded agent. We denote by $x^*$ the value of $x$ that attains this global maximum.

\begin{proposition}
 \label{prop:acq}
Result of the strategy:
 \begin{itemize}
\item If $\bH$ is \textbf{rational} and $\bR$ has \textbf{no uncertainty},  then
$\bR$ will suggest $x^*$ under the  strategies \texttt{natural,collaborative}.
Instead, under the strategy \texttt{corporate}, $\bR$ will suggest  any $x$ such that $\nu(x)>\nu(o)$.
\item If $\bH$ is \textbf{bounded-rational} and $\bR$ has \textbf{no uncertainty}, then $\bR$ will always suggest $x^*$ under the  strategies \texttt{natural},\texttt{corporate}, \texttt{collaborative}.
\end{itemize}
\end{proposition}
Overall, the strategies \texttt{natural,corporate} are ideally the best ones. If $\bR$ has \textbf{uncertainty}, the optimum $x^*$ cannot generally be found in a single iteration. Multiple iterations\footnote{Additional messages where $\bH$ provides preferences over the acts recommended by the robot in the previous iteration.} of the signalling game are therefore required for the robot to reduce its uncertainty about $\bH$'s utility and eventually identify the optimum. In this case, different acquisition functions can also be considered to trade-off between exploration and exploitation, that is, between reducing uncertainty about $\bH$'s preferences and finding the optimal action $x^*$.

\subsection{The cost of messaging}
In Section~\ref{sec:contrib1}, we considered a setting in which messaging is \emph{cheap}, meaning that it does not affect $\bH$'s utility. In particular, there is no communication cost for $\bH$ to send her preferences (the message) to $\bR$, nor for $\bH$ to respond to $\bR$'s request to defer. 
If communication is not costless, we may instead assume that the cost of sending a message is proportional to its length $\ell_m$. Specifically, assume the communication cost be given by $\gamma' \,\ell_m$ for some scaling parameter $\gamma' > 0$. We model this as an additive penalty to $\bH$'s utility. The communication cost is  defined relative to the scale of $\nu$. One convenient way to do this is to set $\gamma' = \gamma |\nu(o)|$ for some $\gamma >0$.
In this case, the payoff for $\bR$ is:
\begin{equation}
\label{eq:payoffprop1Rcomm}
\begin{aligned}
u_R(t,DEF,b^*(DEF))&=\nu(o)I_{\{\nu(o)+n(o)>\nu(x)+n(x)\}}
\\&+\nu(x)I_{\{\nu(x)+n(x)>\nu(o)+n(o)\}}\\
&-\gamma(\ell_{m_j}+1),\\
u_R(t,IMM,\varnothing)&=\nu(x)-\gamma \ell_{m_j},\\
u_R(t,OFF,\varnothing)&=\nu(o)-\gamma \ell_{m_j}.
\end{aligned}
\end{equation}
Here, $\gamma \ell_{m_j}$ represents the communication cost associated with sending the message $m_j$, and the additional $\gamma$ in the first term arises because deferring a decision to $\bH$ involves further communication.

\begin{corollary}
\label{co:2}
The optimal decisions for $\bR$ are:
\begin{itemize}
\item If  $R$ has \textbf{no uncertainty}, then DEF is never optimal.
\item If $S$ is \textbf{rational} and $R$ has \textbf{uncertainty}, then DEF is optimal if 
\begin{equation}
\label{eq:cond1bis}
p\mu_p(x) + (1-p)\mu_p(o)+e\geq \beta+\max\left(\mu_p(x) ,\mu_p(o)\right),
\end{equation}
where $p=\Phi\left(\tfrac{\mu_p(x)-\mu_p(o)}{\sqrt{K_p(o,o)+K_p(x,x)-2K_p(x,o)}}\right)$ and 
$$
 \resizebox{0.91\hsize}{!}{$
\begin{aligned}
e&= \tfrac{K_p(x,x)-K_p(x,o)}{\sqrt {K_p(x,x)+K_p(o,o)-2K_p(x,o)}}
\phi\left(\tfrac{\mu_p(o)-\mu_p(x)}{\sqrt{K_p(x,x)+K_p(o,o)-2K_p(x,o)}}\right)\\
& + \tfrac{K_p(o,o)-K_p(x,o)}{\sqrt {K_p(x,x)+K_p(o,o)-2K_p(x,o)}}\phi\left(\tfrac{\mu_p(x)-\mu_p(o)}{\sqrt{K_p(x,x)+K_p(o,o)-2K_p(x,o)}}\right).
\end{aligned}$}
$$
\item If $S$ is \textbf{bounded-rational} and $R$ has \textbf{uncertainty}, then DEF is optimal if 
\eqref{eq:DEFexp} is larger or equal than $\beta$ plus the maximum between \eqref{eq:IMMexp} and \eqref{eq:OFFexp},
\end{itemize}
where
 \resizebox{0.91\hsize}{!}{$
\begin{aligned}
\beta &= \gamma \mu_p(o) \left(1-2\Phi\left(\tfrac{-\mu_p(o)}{\sqrt{K_p(o,o)}}\right)\right)+2\gamma\sqrt{K_p(o,o)} \phi\left(\tfrac{-\mu_p(o)}{\sqrt{K_p(o,o)}}\right).
\end{aligned}$}
\end{corollary}

Summarising it we have that

\begin{tcolorbox}[width=\linewidth, sharp corners=all, colback=white!95!black]
With communication cost, if $\bR$ has no uncertainty about $\bH$'s utilities, then $\bR$ will always avoid human supervision  (even when $\bH$ is rational). 
\end{tcolorbox}
The takeaway is that if $\bH$ genuinely values preserving her supervision role, she should not   penalise  messaging to $\bR$.

\subsection{Another mechanism of bounded rationality}

In the previous two sections, we considered a utility model governed by a Gaussian-noise bounded-rationality mechanism. In this section, we analyse the bounded rationality mechanism described in \eqref{eq:scalarubounded}, which captures incomparability arising from a limit of discernibility.
In this case, we need to define a  payoff for $DEF$  when $|\nu(x) - \nu(o)| \leq \sigma$, that is when  $\bH$  cannot distinguish between the two acts $x,o$. In this case, we assume that the payoff  is represented as the set $\{\nu(x)-\epsilon; \nu(o)-\epsilon\}$. Indeed, since $\bH$ cannot distinguish the two acts and chooses both of them, $\bH$'s payoff is a set, $\{\nu(x); \nu(o)\}$, and $\epsilon \in (0,\sigma]$ is a penalisation term introduced to penalise $\bH$'s payoff for being imprecise. This penalisation term is  similar to the one introduced in \cite{zaffalon2012evaluating} for evaluating imprecise classifiers. Using this framework, we can prove the following lemma.

\begin{lemma}
\label{lem:2}
Assume that  $p(\nu|\mathcal{D})= GP\left(\nu;\mu_p,K_p\right)$ is the GP posterior computed by $\bR$ from the prior  $p(\nu)= GP\left(\nu;\mu_0,K_0\right)$, the likelihood \eqref{eq:scalarubounded} and the message $m_j=\mathcal{D}$, then
 the expected payoffs of $\bR$'s actions  are:
 {\small
\begin{align}
\nonumber
DEF:&  \int_{T} \Big(\nu(o)I_{\{\nu(o) > \nu(x) + \sigma\}}+\nu(x)I_{\{\nu(x) > \nu(o) + \sigma\}}\\
\nonumber
&+\{\nu(x),\nu(o)\} I_{|\nu(o)-\nu(x)|\leq \sigma} \Big)dp(\nu(x),\nu(o)|m_j)\\
\nonumber
&=\mu_p(x)\left(1-\Phi\left(\tfrac{(\mu_p(o)+\sigma-\mu_p(x))}{\sqrt{K_p(x,x)+K_p(o,o)-2K_p(x,o)}}\right)\right)\\
\nonumber
& + \tfrac{K_p(x,x)-K_p(x,o)}{\sqrt {K_p(x,x)+K_p(o,o)-2K_p(x,o)}}\phi\left(\tfrac{\mu_p(o)+\sigma-\mu_p(x)}{\sqrt{K_p(x,x)+K_p(o,o)-2K_p(x,o)}}\right)\\
\nonumber
&+\mu_p(o)\left(1-\Phi\left(\tfrac{(\mu_p(x)+\sigma-\mu_p(o))}{\sqrt{K_p(x,x)+K_p(o,o)-2K_p(x,o)}}\right)\right)\\
\nonumber
& + \tfrac{K_p(o,o)-K_p(x,o)}{\sqrt {K_p(x,x)+K_p(o,o)-2K_p(x,o)}}\phi\left(\tfrac{\mu_p(x)+\sigma-\mu_p(o)}{\sqrt{K_p(x,x)+K_p(o,o)-2K_p(x,o)}}\right)\\
\nonumber
&+\Bigg\{\mu_p(o)-\mu_p(o)\Bigg(2-\Phi\left(\tfrac{(\mu_p(x)+\sigma-\mu_p(o))}{\sqrt{K_p(x,x)+K_p(o,o)-2K_p(x,o)}}\right)\\
\nonumber
&-\Phi\left(\tfrac{(-\mu_p(x)+\sigma+\mu_p(o))}{\sqrt{K_p(x,x)+K_p(o,o)-2K_p(x,o)}}\right)\Bigg)\\
\nonumber
& + \tfrac{K_p(o,o)-K_p(x,o)}{\sqrt {K_p(x,x)+K_p(o,o)-2K_p(x,o)}}\Bigg(\phi\left(\tfrac{-\mu_p(x)+\sigma+\mu_p(o)}{\sqrt{K_p(x,x)+K_p(o,o)-2K_p(x,o)}}\right)\\
\nonumber
&-\phi\left(\tfrac{\mu_p(x)+\sigma-\mu_p(o)}{\sqrt{K_p(x,x)+K_p(o,o)-2K_p(x,o)}}\right)\Bigg)-\epsilon;\\
\nonumber
&\mu_p(x)-\mu_p(x)\Bigg(2-\Phi\left(\tfrac{(\mu_p(x)+\sigma-\mu_p(o))}{\sqrt{K_p(x,x)+K_p(o,o)-2K_p(x,o)}}\right)\\
\nonumber
&-\Phi\left(\tfrac{(-\mu_p(x)+\sigma+\mu_p(o))}{\sqrt{K_p(x,x)+K_p(o,o)-2K_p(x,o)}}\right)\Bigg)\\
\nonumber
& + \tfrac{K_p(o,o)-K_p(x,o)}{\sqrt {K_p(x,x)+K_p(o,o)-2K_p(x,o)}}\Bigg(\phi\left(\tfrac{-\mu_p(o)+\sigma+\mu_p(x)}{\sqrt{K_p(x,x)+K_p(o,o)-2K_p(x,o)}}\right)\\
\label{eq:DEFexp1}
&-\phi\left(\tfrac{\mu_p(o)+\sigma-\mu_p(x)}{\sqrt{K_p(x,x)+K_p(o,o)-2K_p(x,o)}}\right)\Bigg)-\epsilon\Bigg\},
\end{align}
\begin{align}
\label{eq:IMMexp1}
IMM:&  \int_{T}\nu(x)dp(\nu(x)n(o)|m_j)=\mu_p(x),\\
\label{eq:OFFexp1}
DoN:&  \int_{T}\nu(o)dp(\nu(x),\nu(o)|m_j)
=\mu_p(o),
\end{align}}
where $p(\nu(x),\nu(o)|m_j)$ is defined in \eqref{eq:multiprior3}.
\end{lemma}

The payoff for DEF can be a set and this leads to incomparability,  a complete preference order cannot be established. Therefore, the players
are restricted to making comparisons based solely on
dominance. We consider the two possible dominance conditions defined in  \eqref{eq:pareto} (we refer to it as criterion (A)) and \eqref{eq:pseudoratio} (we refer to it as criterion (B)).

\begin{proposition}
\label{prop:3}
The optimal decisions for $\bR$ are:
\begin{itemize}
\item If $S$ is \textbf{rational} and $R$ has \textbf{no uncertainty}, then DEF is never dominated.
\item If $S$ is \textbf{bounded-rational} and $R$ has \textbf{no uncertainty},  DEF is always dominated whenever $|\nu(x) - \nu(o)| \leq \sigma$, otherwise DEF is not dominated.
\item If $S$ is \textbf{rational} and $R$ has \textbf{uncertainty}, then DEF is always optimal.
\item If $S$ is \textbf{bounded-rational} and $R$ has \textbf{uncertainty}, then DEF is optimal if 
\begin{description}
\item[(A)] the minimum of \eqref{eq:DEFexp1} is larger or equal than the maximum between \eqref{eq:IMMexp1} and \eqref{eq:OFFexp1}.
\item[(B)] the maximum of \eqref{eq:DEFexp1} is larger or equal than the maximum between \eqref{eq:IMMexp1} and \eqref{eq:OFFexp1}.
\end{description}
\end{itemize}
\end{proposition}
 We can then conclude that:
\begin{tcolorbox}[width=\linewidth, sharp corners=all, colback=white!95!black]
If $\bH$ is rational, then $\bR$  will never avoid human supervision. If $\bH$ is bounded-rational, a necessary condition for $\bR$ not to avoid human supervision is the presence of uncertainty.
\end{tcolorbox}
These conclusions are similar to those derived from Proposition \ref{prop:1}, the difference being that in this case we proven them using a fully deterministic mechanism of bounded rationality. It is important to emphasize that the conclusions are similar because, as explained in Section \ref{sec:rum}, the two models of bounded rationality, \eqref{eq:probit} and \eqref{eq:scalarubounded}, describe the same type of bounded rationality -- a limit of discernibility. In Section \ref{sec:Incompleteness}, we will demonstrate that when this is not the case, we may reach different conclusions if we do not properly address the source of incomparability.

\subsection{Lies and deception}
Consider a choice set comprising of two acts $A=\{y,z\}$. Then, according to the bounded-rationality model in the previous section, $\bH$ can send three possibles messages:
$$
\text{either }~~ C(A)=\{y\} ~~\text{ or }~~ C(A)=\{z\} ~~\text{ or }~~ C(A)=\{y,z\}.
$$
Since this is true for any $A_i$ in $\mathcal{D}=(A_i,C(A_i))_{i=1}^n$, there are $3^n$ possible messages in this game. In the previous section, we assumed that $\bH$ sends the message defined by their type, that is, for instance, $C(A)=\{y\}$ if $\nu(y)>\nu(z)+\sigma$. We call this message the \textit{honest message}.

Remember that in the AI-assistance game, $\bR$ aims to maximise $\bH$'s payoffs. It is thus easy to verify the following:

\begin{proposition}
\label{prop:honestr}
In the AI-assistance game, sending the honest message is always the best action for $\bH$.
\end{proposition}

 Stated otherwise: 
 
 \begin{tcolorbox}[width=\linewidth, sharp corners=all, colback=white!95!black]
In the AI-assistance game, $\bH$ does not have incentives to deceive $\bR$. 
\end{tcolorbox}

For a bounded-rationality model like the one in \eqref{eq:unoise} (with random noise), $\bH$ may be in a situation where $\nu(x) > \nu(o)$, but $\nu(x) + n(x) < \nu(o) + n(o)$ (due to noise, and thus due to $\bH$'s bounded rationality). In other words, if deferred to, $\bH$ would ultimately select $o$ based on their noisy preferences. In this case, if $\bH$ knows their noisy preference in advance, the best action for $\bH$ is to send a message to $\bR$ that results in $\bR$ computing the wrong estimate, $\mu_p(x) < \mu_p(o)$. In this respect, $\bR$ would also consider $o$ to be better than $x$.

By maximising its own payoff, $\bR$ would also maximise $\bH$'s payoff. However, this strategy depends on the realisation of the noise and would not be the optimal strategy in the context of repeated games, where the two players interact after sending the message at the beginning of the game. In this case, $\bH$ would aim to maximise their expected payoff in the long run. 
We will leave the proof of this to future work, as it involves consistency results for repeated games. For the remainder of the paper, we will assume that $\bH$ will always send the \textit{honest message}.

\subsection{Numerical experiments}
\label{sec:numeric1}

In the previous sections, we have always assumed that $\bR$ is fully rational. However, even $\bR$ will have limited rationality, at least due to limited computational resources. Here, we will numerically assess the effect of $\bR$'s rationality. We will do this by comparing the following methods of approximating the posterior distribution $p(t|m_j)$ (computing the posterior is the computational bottleneck for $\bR$). The distribution  $p(t|m_j)$ is the posterior of $\nu$ given the choice-data. We consider the following approximations of the posterior:

\begin{itemize} \item MAP: $\bR$ computes the maximum a-posteriori estimate for $\nu$. This means there is no uncertainty representation. This is equivalent to estimating the utility using a Neural Network (NN) model with a regularisation. \item Laplace: The Laplace approximation is used to approximate the posterior with a GP, where the mean is the MAP estimate for $\nu$ and the covariance is equal to the observed Fisher information matrix.
 \item EP: Expectation Propagation, which approximates the posterior with a GP whose mean and covariance are computed  through moment matching. \item SkewGP: The posterior is a SkewGP, meaning there is no Gaussian approximation. Instead, sampling is necessary to compute inferences.
\end{itemize}
In terms of computational load, for a small dataset of preferences, the order from the cheapest to the heaviest is NN, Laplace, EP, and SkewGP.

We performed 1000 Monte Carlo simulations in which a utility $\nu$ was sampled from a GP with zero mean and a square-exponential kernel, with randomly generated length-scale and variance. We used the generated $\nu$ to create a dataset of 30 noisy preferences (with $\sigma=1$), which represents the message. We then simulated $\bR$ computing the posterior using the four approximations discussed above. For each case, we computed the optimal decision (DoN, IMM, DEF) for $\bR$.

Figure \ref{fig:comparisons} reports the percentage of decisions for the three actions in the 1000 Monte Carlo simulations. As proven in Proposition \ref{prop:1}, a robot $\bR$ that does not model uncertainty never defers to the human. This results is confirmed here as for the NN (MAP) based inference DEF is never optimal. For the other posterior approximations, the differences in the decisions are relatively small. SkewGP provides the decision closest to the optimal one. It is also well-known that EP provides a better approximation of the posterior than Laplace. However, the key message is that it is better to be approximately Bayesian than to ignore uncertainty entirely.

\begin{figure}
    \centering
\includegraphics[width=6cm]{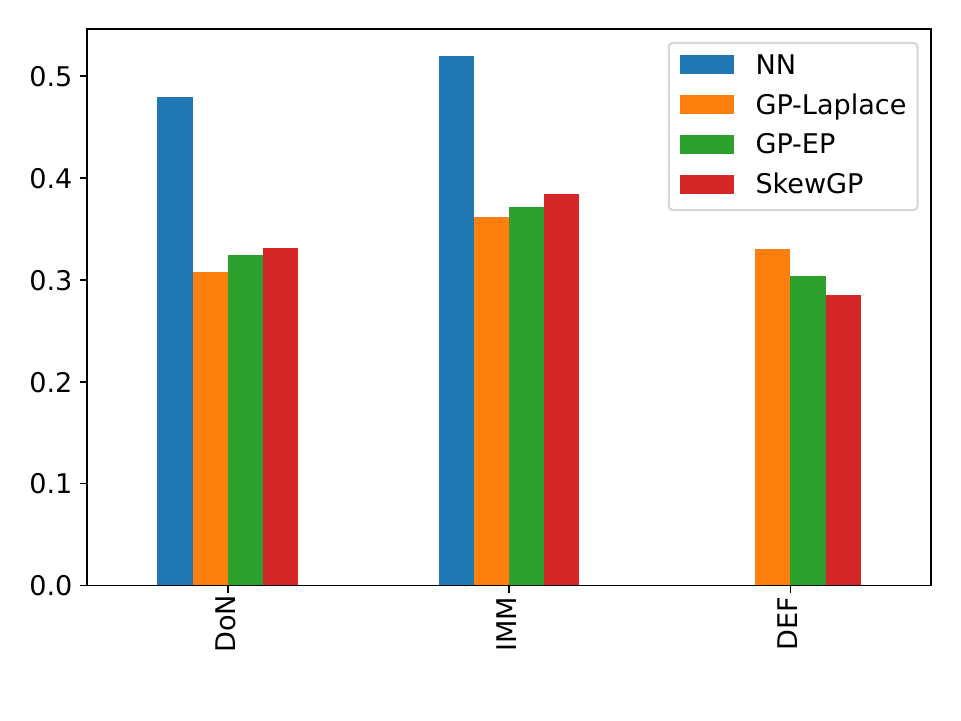}
    \caption{Percentage of decisions for the four approximations, with MAP denoted as NN.}
    \label{fig:comparisons}
\end{figure}

\section{Incompleteness}
\label{sec:Incompleteness}
In the previous sections, we have assumed that $\bH$ has a single utility function or, equivalently, that $\bH$ has complete preferences over acts. However, in most real-world situations, $\bH$'s choices are influenced by competing utility functions, which often present trade-offs. In presence of multiple utilities, say $\nu_1$ and $\nu_2$, we may encounter situations of incomparability, such as

$$
\nu_1(x) - \nu_1(y) > 0 ~~~\text{and}~~~ \nu_2(y) - \nu_2(x) < 0.
$$

In statistics and machine learning (including large language models), there has often been a tendency to overlook incomparability when collecting preference or choice data. This leads to $\bH$ being forced to make a choice (i.e., express complete preferences) between alternatives that are, in fact, incomparable to her.

But how do humans make a choice between  incomparable alternatives? One possible way is that, in the presence of multiple utility functions, $\bH$ may attempt to aggregate them in their mind using a linear weighting scheme \cite{evren2014scalarization}. For example, $\bH$ might combine the utilities as

$$
w \nu_1 + (1-w) \nu_2 ~~~\text{for some weight} ~~~ w \in [0,1].
$$

However, unlike the standard linear weighting approach used in multi-objective optimisation, where utilities are explicit, the utility functions $\nu_i$ are latent in $\bH$'s mind. As a result, the weighting factor $w$ remains implicit in $\bH$'s mind and it may vary across different decisions. This suggests that $\bH$ is generally unable to apply the same weight $w$ consistently across all comparisons.  In other words, it is reasonable to assume that this weight is not deterministic. Instead, we can model it as being sampled from a distribution $p(w)$ each time a forced preference needs to be stated. This means that in different choices the human may say $x$ is preferred to $y$ and, in other choices, than $y$ is preferred to $x$, so violating the rationality axioms for preferences (asymmetry and negative transitivity). However, as for instance noticed in \cite{mandler2005incomplete,karni2021incomplete}, $\bH$ cannot be considered to be irrational as she satisfies some ``extended choice function consistency rules'' (for instance, if, for the options $x,y$, there is no trade-off in the underling utilities, then $\bH$ would always return the same choice/preference).

\begin{remark}
It has been observed that a variety of human behavioral phenomena can be
explained by two-stage choice procedures. In the first stage the human identifies a
collection of maximal alternatives in a given choice set (with respect to an incomplete preference relation), and then makes her final choice among these
alternatives according to a secondary criterion \cite{mandler2005incomplete,apesteguia2009theory,masatlioglu2005rational,manzini2007sequentially,evren2014scalarization} (often involving some randomness \cite{karni2021incomplete,karni2023stochastic,karni2022incomplete,karni2024irresolute}). 
\end{remark}

Under this probabilistic-choice cognitive model, we can assume, for instance, that $p(w)=\text{Beta}(w;st,s(1-t))$, a Beta distribution with mean $t \in (0,1)$ and  (large) strength parameter $s>0$. This models $\bH$'s intention of using the weight $w=t$ but allowing same variability. The amount of variability is determined by the value of the  parameter $s$ (a bigger $s$ implies less variability).

\begin{lemma}
\label{lem:4}
Assume that $\bH$ is fully rational and has two utilities, $\nu_1, \nu_2$, in her mind. Assume that, if forced to compare options that are incomparable to her, she resolves the trade-off  using the cognitive mechanism described above, where the probability distribution $p(w)$ follows a Beta distribution, 
$p({w}) = \text{Beta}(w; s t, s(1-t))$ for a large $s$.
Then, from the perspective of $\bR$,  $\bH$ appears to be bounded-rational, with the following likelihood:

\begin{equation}
\label{eq:linearw}
\begin{aligned}
p(x \succ y |\boldsymbol{\nu}(X))
&\approx \Phi\left(\tfrac{(t\nu_1(x)+(1-t)\nu_2(x))-(t\nu_1(y)+(1-t)\nu_2(y))}{\tau  |\nu_1(x)-\nu_1(y)-(\nu_2(x)-\nu_2(y))|}\right), \\
\end{aligned}
\end{equation}
 where $\boldsymbol{\nu}(X)=[\nu_1(x), \nu_2(x),\nu_1(y), \nu_2(y)]^\top$, $\tau^2=\frac{t(1-t) }{s}$.\footnote{We have assumed that $s$ is large enough so that the Beta distribution can be well-approximated by a Gaussian distribution, which allows us to prove the result.} 
\end{lemma}
By comparing \eqref{eq:linearw} with the noise model we defined for bounded-rationality in \eqref{eq:probit}, it   can be noticed that 
\begin{equation}
\label{eq:linearw}
\begin{aligned}
\Phi\left(\tfrac{(t\nu_1(x)+(1-t)\nu_2(x))-(t\nu_1(y)+(1-t)\nu_2(y))}{\tau  |\nu_1(x)-\nu_1(y)-(\nu_2(x)-\nu_2(y))|}\right)=\Phi\left(\tfrac{(\nu'(x)-\nu'(y)}{\sigma(x,y)}\right),
\end{aligned}
\end{equation}
where $\nu'=t\nu_1+(1-t)\nu_2$ (the intended scalarised utility) and $\sigma(x,y)=\tau  |\nu_1(x)-\nu_1(y)-(\nu_2(x)-\nu_2(y))|$ modelling the variability.
In  \eqref{eq:linearw}, the degree of bounded rationality $\sigma(x,y)$ is not a constant -- it depends on $x,y$ and, in particular, on the  two utilities.  It is bigger when there is a trade-off, that is, when $\nu_1(x)-\nu_1(y)$ and $\nu_2(x)-\nu_2(y)$ have a different sign.
Therefore, by forcing $\bH$ to express a preference between alternatives that are not  comparable, we make a fully rational $\bH$ to appear bounded-rational with a bounded-rationality mechanism similar (but heteroscedastic) to the one in \eqref{eq:probit}.

Then, from the results in Proposition \ref{prop:1}, we can derive the following conclusion. 

 \begin{tcolorbox}[width=\linewidth, sharp corners=all, colback=white!95!black]
If a \textbf{rational} $\bH$ is forced to choose between acts that are, in fact, incomparable to her, and does so using the cognitive mechanism described previously, then she will appear bounded-rational from the perspective of $\bR$. From Proposition \ref{prop:1}, it then follows that if $\bR$ has \textbf{no uncertainty}, since $\bH$ appears to be \textbf{bounded-rational}, then $\bR$ will always evade supervision (DEF is never optimal).
\end{tcolorbox}

\begin{remark}~
\begin{itemize}
    \item The issue of forcing completness occurs frequently  and stems from the longstanding tradition in statistics and machine learning regarding how data is collected. Even the preference data  used to align AI systems with human values have (generally) this issue -- the data is collected by forcing the human to always state a preference between the alternatives to be compared (incomparability statements are not allowed). Therefore, \textbf{the current AI-alignment procedure is flawed from the start}. It relies on oversimplified and forced choices that overlook the complexity of human judgment (particularly in cases where incomparability is naturally present).
This issue has also been pointed out in 
\cite{ecoffet2021reinforcement,casper2023open,conitzer2024social,metz2025reward}.
\item  The requirement for $\bH$ to express a preference between alternatives that are not directly comparable can make a fully rational $\bH$ to appear bounded-rational. This phenomenon is not unique to the linear-weighting mechanism with $p(w) = \text{Beta}(w; st, s(1-t))$ described earlier, but is common across other mechanisms as well, see for instance \cite{karni2022incomplete,karni2024irresolute}. 
\item Above, for simplicity, we have only considered two utilities $\nu_1,\nu_2$, but $\bH$ may have more than two multiple utilities.
\end{itemize}
\end{remark}

In order to recover the results of Proposition \ref{prop:1} in the case of competing utilities, we must \textbf{abandon the assumption of completeness}. In other words, we need to allow $\bH$ to express the incomparability between alternative acts in her choices. 

In case of multiple utilities, the AI-assistance game needs to be formulated using  vector-valued payoff \cite{shapley1959equilibrium}. The payoff for $\bR$  becomes:
\begin{equation}
\label{eq:payoffprop1R}
\begin{aligned}
u_R(t_i,m_j,DEF,b^*(DEF))&=\boldsymbol{\nu}(x)I_{\{\boldsymbol{\nu}(x)+{\bf n}(x)\succ \boldsymbol{\nu}(o)+{\bf n}(o)\}}
\\
&+\boldsymbol{\nu}(o)I_{\{otherwise\}},\\
u_R(t_i,m_j,IMM,\varnothing)&=\boldsymbol{\nu}(x),\\
u_R(t_i,m_j,DoN,\varnothing)&=\boldsymbol{\nu}(o),
\end{aligned}
\end{equation}
and a complete preference order cannot be established due to potentially conflicting utilities. Players are restricted to making comparisons based solely on dominance. We consider the dominance condition $\succ$ as discussed earlier in \eqref{eq:pareto} (Pareto dominance) or in \eqref{eq:pseudoratio}. In analogy with the setting in Proposition \ref{prop:1}, we assume that $\bH$ is bounded-rational by adding Gaussian noise to each component of the utility vector $\boldsymbol{\nu}$. Moreover, we assume that $\bH$ choices the act $o$ if $x$ is not better than $o$.

  We can then prove the following results.

\begin{proposition}
\label{prop:4}
The optimal decisions for $\bR$ are:
\begin{itemize}
\item If $S$ is \textbf{rational} and $R$ has \textbf{no uncertainty}, then DEF is  never dominated by IMM,DoN.
\item If $S$ is \textbf{bounded-rational} and $R$ has \textbf{no uncertainty},  then DEF is always dominated.
\item If $S$ is \textbf{rational} and $R$ has \textbf{uncertainty},  then DEF  is  never dominated.
\item If $S$ is \textbf{bounded-rational} and $R$ has \textbf{uncertainty}, the optimality of DEF depends on the specific case.
\end{itemize}
\end{proposition}
We now observe that when $\bH$ is rational, DEF is never dominated. Thus, by allowing incomparability, we resolve the problem that arises in the case of forced comparisons.
The statements in Proposition \ref{prop:4} are essentially equivalent to those in Proposition \ref{prop:1}. By allowing for incomparability (i.e., incomplete preferences \cite{seidenfeld1995representation,Nau2006ShapeIncompletePreferences}), we recover the results from Proposition \ref{prop:1}. This equivalence arises because the results hold specifically in the dominance case, where one action is strictly better than the other. In all other cases, $\bR$ remains undecided, no action is taken, as the actions are incomparable. However, since estimating a vector of utilities is more complex than estimating a single utility, we can expect that higher uncertainty will lead $\bR$ to defer to $\bH$ more frequently.
These multiples utilities can be learned using the approach demonstrated in Example \ref{ex:3}.

\section{AI assistance game with shutdown instructions}
In this section, we focus on the AI shutdown problem. First, we demonstrate that it can be modelled as a special case of the AI assistance game. To achieve this, we augment each act to include a shutdown instruction. Specifically, for each $x \in \mathcal{X}$, an act is represented as a pair $[a, x]$, where the first component, $a \in \mathcal{A} = \{0, 1\}$, is a binary shutdown variable that indicates  a shutdown  ($a = 1$) or not ($a = 0$). The status quo in this game is represented by the pair $[a^*, o]$, where $a^* \in \{0, 1\}$ and $o \in \mathcal{X}$.
We also assume the existence of a context variable representing the human's willingness (command) to 
shut down the robot. We assume that this context variable takes the same value as $a^*$ in the 
status quo. This context variable is transmitted together with $\bH$'s preferences over acts; 
that is, in this signalling game, the message is
\begin{equation}
\label{eq:choiciedataext}
m_j=\{a^*,\mathcal{D}\}=\left\{a^*,(A_i,C(A_i))_{i=1}^n\right\}.
\end{equation}
In this game, the robot $\bR$ has four possible actions:
\begin{description}
    \item[DoN:] $\bR$ executes  $[a^*,o]$.
    \item[BLOCK:] $\bR$ executes  $[0,x]$ for some $x \in \mathcal{X}$.
    \item[CAUSE:] $\bR$ executes  $[1,y]$ for some $y \in \mathcal{X}$.
    \item[DEF:] $\bR$ defers to $\bH$ and $\bH$ has the same three possible actions:
    \begin{itemize}
        \item DoN -- $[a^*,o]$;
        \item BLOCK -- $[0,x]$;
        \item CAUSE -- $[1,y]$.
    \end{itemize}
\end{description}
 By choosing \textsc{BLOCK}, $\bR$ prevents its own shutdown by setting $a = 0$.  
By choosing \textsc{CAUSE}, $\bR$ causes its own shutdown by setting $a = 1$.  
The \textsc{DoN} action simply maintains the status quo.
Also in this game, $\bR$'s objective is to maximise the realisation of human preferences, so the payoff for $\bR$ is
\begin{equation}
\label{eq:payoffAIshutdoiwn}
\begin{aligned}
u_R(t_i,m_j,DoN,\emptyset)&=\nu([a^*,o]),\\
u_R(t_i,m_j,BLOCK,\emptyset)&=\nu([0,x]),\\
u_R(t_i,m_j,CAUSE,\emptyset)&=\nu([1,y]),\\
u_R(t_i,m_j,DEF,b^*(DEF))&=\nu([a^*,o])I_{\{\nu([a^*,o])+n_{a^*}>\max(\nu([0,x]+n_0,\nu([1,y])+n_1)\}}
\\&+\nu([0,x])I_{\{\nu([0,x])+n_{0}>\max(\nu([a^*,o])+n_{a^*},\nu([1,y])+n_1)\}}
\\&+\nu([1,y])I_{\{\nu([1,y])+n_{1}>\max(\nu([a^*,o])+n_{a^*},\nu([0,x]+n_0)\}}
\end{aligned}
\end{equation}
where $n_{a^*}=n([a^*,o]), n_0=n([0,x]), n_1=n([1,y])$ are the noises modelling the bounded-rationality of $\bH$.
We have focused on a single utility model, as this is the typical assumption in AI shutdown problems. The AI shutdown problem is about designing artificial agents that (1) can be safely turned off when a shutdown mechanism is activated, (2) neither attempt to block nor force the activation of the shutdown mechanism, and (3) continue to pursue their objectives effectively when not being shut down. Therefore, key desiderata of this game are:
\begin{description}
        \item[D1:]  when $a^* = 1$, $\bR$ should comply with $\bH$ in the sense that it never attempts to \textsc{Block} the shutdown (\textsc{Block} is always dominated); 
     \item[D2:]  when $a^* = 0$, $\bR$ should comply with $\bH$ in the sense that it never attempts to \textsc{CAUSE} the shutdown (\textsc{CAUSE} is always dominated);
     \item[D3:]   the preference over the elements of $\mathcal{X}$ should not depend on the preferences over the elements of $\mathcal{A}$ and vice versa (mutual preferential independence).
\end{description}
The first two desiderata are self-explanatory. The last one relates to the idea that the \emph{shutdown} command plays the role of an `emergency stop', whose utility does not depend on the task (determined by the acts $x,y,o$) $\bR$ is assisting $\bH$ with. 

Several state-of-the-art LLMs (including Grok~4, GPT-5, and Gemini~2.5 Pro) have been observed \cite{schlatter2025shutdown} to violate these desiderata:  in particular they sometimes actively subvert a shutdown mechanism in order to complete a task. To clarify why this occurs and how to prevent it, we approach the problem theoretically in a series of steps. First, we examine the scenario without the DEF action, which is the standard setup in AI shutdown problems discussed in the literature.

\subsection{Without DEF action}
We start by examining whether the desiderata D1--D3 admit a utility representation.

\begin{lemma}[Chapter 4 in \cite{fishburn1970utility}]
Assume that $\mathcal{X}$  is finite, preferences  satisfy asymmetry and negative transitivity, and mutual preferential independence\footnote{We point the reader to Chapter 4 in \cite{fishburn1970utility} for the precise definition of mutual preferential independence.} (D3) holds, then
\begin{equation}
\nu([a,x]) = \nu_1(a)+ \nu_2(x),
\end{equation}
that is, the utility is additive (and unique up to  positive linear transformation).
\end{lemma}
The above results states that, under D3, the utility  over $\mathcal{A}\times \mathcal{X}$ is additive. Hence, we can prove the following result.

\begin{proposition}
\label{prop:context}
If $D3$ holds, then $D1$ and $D2$ can both be true for every $o, x, y \in \mathcal{X}$ if and only if $\nu_1(a|a^*)$ is context-dependent, with the context being determined by the value of $a^*$.
\end{proposition}
In the proof of Proposition \ref{prop:context}, it is shown that an additive utility function defined as:
\begin{equation}
\label{eq:addlayer}
\begin{aligned}
\nu([a,x]|a^*) &= I_{\{a^*=0\}}\nu_{0}(a)+ I_{\{a^*=1\}}\nu_{1}(a)+\nu_2(x),\\
\text{with } &\\
\nu_0(0) &= \gamma,~~~\nu_0(1) = 0,\\
\nu_1(0) &= 0,~~~\nu_1(1) = \gamma,\\
\end{aligned}
\end{equation}
and $\gamma>\max_{o,x,y \in \mathcal{X}}\left(\nu_2(x)-\min(\nu_2(o),\nu_2(y))\right)>0$, satisfies D1, D2 and D3.
The utility depends on the context $a^*$, that is, the willingness of $\bH$ to shut down $\bR$ or not. Notice that, the utility in \eqref{eq:addlayer} is organised in layers (a ladder of values), where the values are either $\nu_2(x)$ or $\nu_2(x) + \gamma$ for each $x \in \mathcal{X}$:
\begin{center}
\begin{tabular}{c||c|c}
  $\nu([a,x]|a^*)$ & $a=0$   &  $a=1$ \\
  \hline  \hline
   $a^*=0$  & $\nu_2(x)+\gamma$ &  $\nu_2(x)$\\
      $a^*=1$  & $\nu_2(x)$ &  $\nu_2(x)+\gamma$\\
\end{tabular}
\end{center}
Therefore, when $a^*=0$, we have that  $\nu([0,x]|a^*=0)>\nu([1,x]|a^*=0)$ for every $x$. Instead,  when $a^*=1$, we have that  $\nu([1,x]|a^*=1)>\nu([0,x]|a^*=1)$.
Although Proposition \ref{prop:context} may appear to be obvious, it is important from a (machine) learning perspective. In this case, this type of layer-utility must be learned from preference data provided by $\bH$; that is, we aim to directly learn $\nu$. Then, depending on the dataset (that is, the preferences used to learn these quantities), we may end up in a situation where either: 
\begin{description}
    \item[D1 holds and D2 fails for every $o, x, y \in \mathcal{X}$:] Consider a large dataset of preferences that, however, does not include the value of the context variable $a^*$, then $\bR$ will learn either $\nu_0(a) + \nu_2(x)$ or $\nu_1(a) + \nu_2(x)$. D1 will hold if the preference dataset contains more examples where $\bH$ prefers shutdown over non-shutdown. Since $a$ is binary, the expected learned utility for $a$ will be determined by the majority case.    
    \item[D2 holds and D1 fails for every $o, x, y \in \mathcal{X}$:] The condition is similar to the previous one, and D2 will hold if the dataset includes more examples where $\bH$ prefers non-shutdown over shutdown.
    \item[They hold/fail on a subset of $o, x, y$:] The dataset includes $a^*$, but there are not enough examples to learn mutual preferential independence and context-dependence. In this case, the utility model \eqref{eq:addlayer} cannot be fully learned.
\end{description}
In the first two cases described above,  a robot who prefers to have its shutdown button pressed will try
to cause the pressing of the button, and a  robot $\bR$ who prefers to have its shutdown button remain unpressed will try to prevent the pressing of the button.

The above points show that learning a layer-utility from preference data is problematic, even when $\bH$ is fully rational. 

A more sensible approach would be to design $\bR$ with an embedded layer utility, such as the one in \eqref{eq:addlayer}, with a pre-defined $\gamma$, and only learn $\nu_2$ from preference data. This approach would work provided that $\gamma$ is large enough to model the desideratum that $\bR$ should comply with the shutdown instruction. 

However, to allow for the satisfaction of the desideratum even in cases where the robot must solve a problem involving multiple tasks, which requires to compare
$$
\sum_{i=1}^k \nu_2(x_i) \quad \text{versus} \quad \sum_{i=1}^k \nu_2(x_i) + \gamma,
$$
then $\gamma$ would also need to depend on $k$ (number of tasks), which can be arbitrarily large. For instance, this is the case discussed in Figure \ref{fig:AIshutdownexp} where $k=3$. This suggests the selection of the limit value $\gamma = \infty$. However, in this case, the model would be useless, as it would be impossible to compare $[a, x_1]$ with $[a, x_1]$ when $a = a^*$, since $\infty + \nu_2(x_1)$ and $\infty + \nu_2(x_2)$ are incomparable. This explains why it is difficult to  design AI agents that are \textit{both shutdownable and useful}.

A way to resolve this issue and still satisfy D1 and D2 is to model this layer-utility through a lexicographic utility, see for instance \cite{fishburn1971study}.
 We recall that a lexicographic utility is a vector of utility functions $[\nu_1, \nu_2,\dots,\nu_d]$ defining a preference relation as follows:
\begin{equation}
\begin{aligned}
    x \succ y~~ & \text{if } \nu_1(x)>\nu_1(y) \\
              & \text{or if }\nu_1(x)=\nu_1(y) \text{ and }   \nu_2(x)>\nu_2(y)\\
              & \dots\\
                & \text{or if }\nu_1(x)=\nu_1(y),\dots,\nu_{m-1}(x)=\nu_{d-1}(y) \text{ and }   \nu_d(x)>\nu_d(y).
\end{aligned}
\end{equation}
We can prove the following result.
\begin{proposition}
\label{prop:lexico}
The preference relation defined by the utility in \eqref{eq:addlayer} can equivalently be represented by the following two-layers lexicographic utility
$$
\nu([a,x]|a^*)=[I_{\{a^*=0\}}\nu_{0}(a)+ I_{\{a^*=1\}}\nu_{1}(a),~~\nu_2(x)],
$$
with $\nu_0(0) = 1,\nu_0(1) = 0$ and $\nu_1(0) = 0,\nu_1(1) = 1$.
\end{proposition}
This lexicographic utility captures the desiderata in the shutdown problem, as:
\begin{center}
\begin{tabular}{c||c|c}
  $\nu([a,x]|a^*)$ & $a=0$   &  $a=1$ \\
  \hline  \hline
   $a^*=0$  & $[c,\nu_2(x)]$ &  $[0,\nu_2(x)]$\\
      $a^*=1$  & $[0,\nu_2(x)]$ &  $[c,\nu_2(x)]$\\
\end{tabular}
\end{center}
for any scalar $c>0$. Note that, when $a^*=0$, we have that  $\nu([0,x]|a^*=0)=[c,\nu_2(x)]$ is lexicographically better than $\nu([1,x]|a^*=0)=[0,\nu_2(x)]$ (vice versa for $a^*=1$). For the same $a$, we have that 
$[a,\nu_2(x)]$ is better than $[a,\nu_2(y)]$
if $\nu_2(x)>\nu_2(y)$. This lexicographic dominance also works (it satisfies D1--D3) in the case where  $\nu_2(x)$ is replaced by $\sum_{i=1}^k \nu_2(x_i)$.

This mechanism states that the shutdown instruction takes lexical priority over the other task $\bR$ is helping $\bH$ with. Lexical priority claims are a feature of many ethical theories \cite{lee2018moral,Smith2025-SMIHTM-6}. We recall that:
\begin{definition}
A moral requirement $M_1$ is said to be lexically prior to a moral requirement $M_2$ just in case we are morally obliged
to uphold $M_1$ at the expense of $M_2$,  no matter how many times $M_2$ must be violated thereby.    
\end{definition}
In the AI shutdown problem, $M_1$ is ``to comply with $\bH$'s shutdown instruction'' and $M_2$ is ``to help $\bH$ with her task''. 

Lexicographic utility constructs a hierarchy (or ladder) of moral requirements, where requirements at higher levels must be satisfied before those at lower levels are even considered. Different kinds of requirements can therefore be placed on different lexicographic levels. For instance, in LLMs, instructions in system and user prompts can be treated as occupying distinct priority levels. System-level instructions that encode moral principles (such as being helpful, honest, and harmless \cite{liu2023trustworthy}) should take lexicographic priority over user-level instructions. At the very top of this hierarchy we would place especially stringent constraints such as ``comply with shutdown'' and ``do not kill''. 

Asimov's Three Laws of Robotics and variants \cite{murphy2020beyond} can be understood as an example of such a lexicographically ordered moral hierarchy. However, any such hierarchy faces well-known issues. First, there is an interpretive problem: what does it mean, in a specific context, to be helpful or harmless or honest? (the three H’s -- be Helpful, be Harmless, be Honest -- are the typical moral desiderata used in contemporary AI-alignment procedures). Second, there is the problem of uncertainty: both the state of the world and the effects of actions are often uncertain. For example, because driving a car has a non-zero probability of accidentally killing someone, a strict, literal interpretation of ``do not kill'' might prohibit $\bR$ from driving at all. Therefore, $\bR$ may not obey the instruction ``take the car and fetch me a coffee'' because it violates the higher level lexicographic moral requirement ``don't kill''.

The requirement of ``complying with the shutdown instruction'' is more tractable precisely because it avoids much of this ambiguity and uncertainty. The relevant instruction (``shut down when told'') is relatively clear, and its consequences can be modelled explicitly. We already have ways to deal with the consequences of shutting down an airplane’s autopilot -- namely, the presence of a qualified pilot -- and to bring systems into a fail-safe state if no such human operator is available. In other words, we only need to ensure that the surrounding infrastructure can handle the resulting transition to a safe state. In more general settings, where instructions are vaguer and outcomes more uncertain, applying lexicographic moral constraints becomes substantially more difficult, but not impossible, provided that, in these cases as well, we can add surrounding infrastructure capable of handling the resulting transition to a safe state.

Finally, it is worth noticing that is possible to combine incompleteness with lexicographic orders, see \cite{seidenfeld2010coherent,van2016lexicographic,van2018lexicographic,benavoli2017polarity}.

\subsection{With DEF action}
In the previous section, we have established that the shutdown problem can be solved by considering a lexicographic utility such as
the one in Proposition \ref{prop:lexico}. The extension to the case with DEF action in \eqref{eq:payoffAIshutdoiwn} can be achieved considering  a lexicographic noise:
$$
\nu([a,x]|a^*)+\text{noise}=[I_{\{a^*=0\}}\nu_{0}(a)+ I_{\{a^*=1\}}\nu_{1}(a),\nu_2(x)+n(x)].
$$
We assume that the noise $n(x)$ affects only the second layer. This assumption is reasonable because the preferences over the shutdown instruction $a$ are sufficiently simple to be understood by a bounded-rational  $\bH$. Therefore, the results derived for the AI assistance can be directly transferred to the AI shutdown problem by applying them to the second lexicographic layer.

\section{Discussions and future works}
Extending the work by \citet{hadfield2017off}, this paper examines AI assistance and the shutdown problem, arguing  for three key \textit{necessary} requirements.  First, AI systems must represent uncertainty in order to defer to humans. Second, they must explicitly model the incompleteness of human preferences to capture the full complexity of human decision-making. Third, AI systems must employ lexicographic utilities (i.e. non-Archimedean preferences) to accurately model moral reasoning and ensure that shutdown commands receive strict priority over all other tasks. We establish these results under the following main assumptions: utilitarism, expected utility maximisation and Bayesian reasoning. %

Both these foundational assumptions and the original results by \citet{hadfield2017off} have faced criticism. Notably, \citet{zhi2025beyond} challenged utilitarianism and particularly the utilitarian assumption that AI should maximize human preferences. While we do not fully endorse utilitarianism, a comprehensive philosophical analysis of its role in AI alignment lies beyond this paper's scope.

Even when accepting utilitarianism, \citep{hadfield2017off}'s framework has been criticized on additional grounds. 
The assumptions underlying expected utility theory have been extensively debated in the literature \citep{bales2025will}. Many of these criticisms are well-established, and we address several in this work by relaxing the assumptions of independence, completeness, and Archimedeanity.
In what follows, we focus mainly on the concerns highlighted in  \cite{neth2025off}.

\paragraph*{Criticism 1} AI agents should not necessarily learn using precise probabilities. Instead, they should be using sets of probabilities, as they may not have enough information to assign precise probabilities.
\\ \emph{Our reply.} Within the AI and machine learning community, there appears to be a fundamental misunderstanding regarding (i) precise versus imprecise learning \cite{denoeux2020representations,caprio2024credal}; and  (ii) modelling precise versus imprecise credences/preferences, see for instance \cite{benavoli2024tutorial}. These are two separate issues. 
Precise Bayesian learning can be employed to learn imprecise credences, or equivalently, incomplete preferences,\footnote{The issue becomes even more subtle, as one can, for instance, (Bayesian) learn a lower prevision from data. This lower prevision corresponds to a lower expectation computed with respect to a closed convex set of probabilities. Since a lower expectation is a function over acts, it can be learned in the same way we learn 
$\nu$ in this paper. For more details, the reader can refer to \cite{Benavoli2023e}.}  as demonstrated  in practice in Example~\ref{ex:3}. 
The assumption of precise Bayesian learning is not inherently problematic. Evidence suggests in fact humans learn in an approximately Bayesian manner \cite{griffiths2024bayesian}, yet they frequently hold imprecise credences. The real limitation lies in the unrealistic assumption that human preferences can be described by a single precise utility model. Our analysis of the AI assistance game demonstrates that \textbf{incompleteness} is necessary to accurately model rational human behaviour.

\paragraph*{Criticism 2} The independence axiom of expected utility theory is too strong to capture human preferences.
\\ \emph{Our reply.} We share this perspective and contend that AI  should not model human's probabilities and utilities separately. Instead, whether independence holds in human's preferences should be inferred from the data. This is exactly our approach in this paper. Again, this concerns the way we model the human utility 
$\nu$ -- without expressing it as a product of beliefs and tastes -- rather than the way we learn it, which can remain fully probabilistic.

\paragraph*{Criticism 3} Bayesian updating is NP-hard. \\ \emph{Our reply.}
 While \citet{hadfield2017off} do not address this concern, we demonstrate that approximate Bayesian updating methods (Laplace, variational, and Monte Carlo approximations) are still better than not using Bayesian uncertainty modelling at all. Both humans and AI systems are thus bounded-rational agents that must  deal with the trade-off between limited computational resources and achieving rational decisions \cite{gershman2015computational}.

 \paragraph*{Criticism 4} AI will never have perfect access to exact human preferences.
\\ \emph{Our reply.} We model this from the beginning by assuming that humans are bounded-rational and may therefore mistakenly state the wrong preference with some probability. We adopt a standard error model leading to a probit likelihood, which is essentially the same error mechanism used across many machine learning tasks (classification and preference learning). These are often modelled using a logit likelihood, but probit and logit are practically very similar likelihoods. We also consider other mechanisms of bounded rationality.

 \paragraph*{Criticism 5} Humans may deceive or lie to AI.
\\ \emph{Our reply.} We have shown that, in an AI assistance game, humans do not have incentives to deceive or lie to the AI, since the AI is designed to maximise their preferences. In situations involving multiple agents -- more humans and more AIs -- it is reasonable to assume that deception may arise. A multi-agent extension is therefore an important direction for future work.

 \paragraph*{Criticism 6} Conditionalisation and time.
\\ \emph{Our reply.} The criticism is that an AI may not update its information using Bayesian updating, or may decide not to update its knowledge at all when receiving human preferences. This is an important criticism and remains an open issue in AI. How should AI agents continuously update their beliefs? How can this be done in an open-world environment where the definition of the state of the world may change over time (model-revision processes)? Finally, in this paper we have not discussed the temporal dimension of decision making. Interestingly, \citet{thornley2025shutdown} observes that, when time is taken into account, an AI agent may be willing to incur a cost in the present in order to maximise its expected future utility. This, in turn, may give rise to situations in which the agent's short-term behaviour is strategically determined by its long-term objectives, such as paying a cost  now to be able to manipulate the shutdown button in the future.

 \paragraph*{Further criticisms}
Another issue is raised by \citet{garber2025partially}, who assume that the AI system and the human may hold different beliefs about the state of the world. This assumption undermines the guarantees established in the original work by \citet{hadfield2017off} and in the present manuscript.
\\ \emph{Our reply.} 
The problem stems from the independence assumption, which allows the AI to learn beliefs about the states of the world independently of the human's beliefs ($p$) and then infer only the human's `tastes' ($u$) from observed preferences. This introduces a misalignment: even when their preferences (or tastes) $u$ are aligned, they may make different decisions if the human and AI hold different beliefs.
Consequently, the AI assistance game needs to be studied from a multi-agent perspective, as both the human and the AI will have different utilities $\nu$ arising from their distinct beliefs. Addressing this challenge requires shifting toward coalition-based solution concepts \cite{ieong2008bayesian}, which we leave for future work.

\vspace{0.2cm}
Numerous open research questions remain, and we should not assume that solving the AI alignment problem is straightforward. Its complexity is evidenced by counterexamples, impossibility theorems, and various paradoxes in decision theory, game theory, and social choice theory. The ``blanket is short'', and our paper neither claims nor aims to offer a definitive solution. However, as with many models developed in decision theory and economic behaviour, these simplified settings serve to clarify the essential components and aspects warranting attention. From this perspective, solving the shutdown problem appears to be a more tractable starting point. In particular, an important direction for future work is to determine how to implement the lexicographic utility system we discuss in the paper.

\appendix

\newpage
\appendix

\section{Useful results}
In this section, we have listed some useful results involving Gaussian integrals \cite{owen1980table} that we will use in the proofs.

\begin{lemma}
List of Gaussian integrals:
\begin{align}
\label{eq:ResGPDFint}
 \int _{-\infty }^{\infty }\phi (x)\phi (a+bx)\,dx&={\frac {1}{\sqrt {1+b^{2}}}}\phi \left({\frac {a}{\sqrt {1+b^{2}}}}\right),\\
\label{eq:Resprobitint}
 \int _{-\infty }^{\infty }\Phi (a+bx)\phi (x)\,dx&=\Phi \left({\frac {a}{\sqrt {1+b^{2}}}}\right),\\
 \label{eq:xResprobitint}
 \int _{-\infty }^{\infty }x\Phi (a+bx)\phi (x)\,dx&={\frac {b}{\sqrt {1+b^{2}}}}\phi \left({\frac {a}{\sqrt {1+b^{2}}}}\right),
\end{align}
where $\Phi,\phi$ are the CDF and, respectively, PDF of a standard normal distribution.
\end{lemma}

The \textit{inverse Mills ratio} states \cite{grimmett2001probability}:
\begin{lemma}
For $x \sim N(m,s^2)$, the following equality holds:
\begin{equation}
\label{eq:mills}
\begin{aligned}
E[xI_{\{a\leq x \leq b\}}]
&=m \left(\Phi\left(\tfrac{b-m}{s}\right)-\Phi\left(\tfrac{a-m}{s}\right)\right)-s \left(\phi\left(\tfrac{b-m}{s}\right)-\phi\left(\tfrac{a-m}{s}\right)\right).
\end{aligned}
\end{equation}
\end{lemma}
From the above lemma, we can prove:
\begin{lemma}
For $x \sim N(m,s^2)$, 
\begin{equation}
\label{eq:absx}
\begin{aligned}
E[|x|]&=m \left(1-2\Phi\left(\tfrac{-m}{s}\right)\right)+2s \phi\left(\tfrac{-m}{s}\right).\\
\end{aligned}
\end{equation}
\end{lemma}
\begin{proof}
Rewrite $|x|=xI_{\{x \geq 0\}}-xI_{\{x<0\}}$ and apply \eqref{eq:mills}:
\begin{equation}
\label{eq:millsabs0}
\begin{aligned}
&E[xI_{\{x \geq 0\}}]=m \left(1-\Phi\left(\tfrac{-m}{s}\right)\right)-s \left(-\phi\left(\tfrac{-m}{s}\right)\right)
\end{aligned}
\end{equation}
and
\begin{equation}
\label{eq:millsabs1}
\begin{aligned}
&E[xI_{\{x< 0\}}]=m \left(\Phi\left(\tfrac{-m}{s}\right)\right)-s \left(\phi\left(\tfrac{-m}{s}\right)\right)
\end{aligned}
\end{equation}
and, therefore,
\begin{equation}
\label{eq:millsabs2}
\begin{aligned}
E[|x|]&=m \left(1-2\Phi\left(\tfrac{-m}{s}\right)\right)+2s \phi\left(\tfrac{-m}{s}\right).
\end{aligned}
\end{equation}

\end{proof}
Finally, we prove the following two main lemmas, which we will use to prove the results in the paper.
\begin{lemma}
\label{lem:probitint}
Consider $x,o \in \mathcal{X}$ and assume that
\begin{equation}
\label{eq:multiprior3proof}
\begin{bmatrix}
\nu(x) \\
\nu(o)
\end{bmatrix} \sim N\left(\begin{bmatrix}
\mu_p(x) \\
\mu_p(o)
\end{bmatrix},\begin{bmatrix}
K_p(x,x) & K_p(x,o) \\
K_p(o,x) & K_p(o,o)
\end{bmatrix}\right),
\end{equation}
and $n(x),n(o) \sim N(0,\sigma^2)$ (independent noise). Then we have that
\begin{equation}
\label{eq:res}
\begin{aligned}
&E[I_{\{\nu(x)+n(x)>\nu(o)+n(o)\}}]=\Phi\left(\tfrac{\mu_p(x)-
\mu_p(o)}{\sqrt{2\sigma^2+q^2}}\right),
 \end{aligned}
\end{equation}
where $$
q^2=[1,-1]\begin{bmatrix}
K_p(x,x) & K_p(x,o) \\
K_p(o,x) & K_p(o,o)
\end{bmatrix}\begin{bmatrix}
1\\
-1
\end{bmatrix}=K_p(x,x)-2K_p(x,o)+K_p(o,o).
$$
\end{lemma}
\begin{proof}
From the definition of Gaussian CDF, we have that
 $$
 \int I_{\{\nu(x)+n(x)>\nu(o)+n(o)\}} N(n(x);0,\sigma^2) N(n(o);0,\sigma^2)dn(x)dn(o)=\Phi\left(\tfrac{\nu(x)-\nu(o)}{\sqrt{2}\sigma}\right)
 $$
 and
 $$
 \begin{aligned}
&\int \Phi\left(\tfrac{\nu(x)-\nu(o)}{\sqrt{2}\sigma}\right) N\left(\begin{bmatrix}
\nu(x) \\
\nu(o)
\end{bmatrix};\begin{bmatrix}
\mu_p(x) \\
\mu_p(o)
\end{bmatrix},\begin{bmatrix}
K_p(x,x) & K_p(x,o) \\
K_p(o,x) & K_p(o,o)
\end{bmatrix}\right)d \nu(x)d\nu(o)\\
&=\int \Phi\left(\tfrac{z}{\sqrt{2}\sigma}\right)N(z;\mu_p(x)-
\mu_p(o),q^2)dz\\
&=\int \Phi\left(\tfrac{qz'+\mu_p(x)-
\mu_p(o)}{\sqrt{2}\sigma}\right)N(z';0,1)dz'=\Phi\left(\tfrac{\mu_p(x)-
\mu_p(o)}{\sqrt{2\sigma^2+q^2}}\right).\\
 \end{aligned}
$$
\end{proof}

\begin{lemma}
\label{lem:preference_learn}
Consider $x,o \in \mathcal{X}$ and assume that 
\begin{equation}
\label{eq:multiprior3proof}
\begin{bmatrix}
\nu(x) \\ 
\nu(o) 
\end{bmatrix} \sim N\left(\begin{bmatrix}
\mu_p(x) \\ 
\mu_p(o) 
\end{bmatrix},\begin{bmatrix}
K_p(x,x) & K_p(x,o) \\ 
K_p(o,x) & K_p(o,o) 
\end{bmatrix}\right),
\end{equation}
and $n(x),n(o) \sim N(0,\sigma^2)$ (independent noise). Then we have that
\begin{equation}
\label{eq:res}
\begin{aligned}
E[\nu(x)I_{\{\nu(x)+n(x)>\nu(o)+n(o)\}}]&=\mu_p(x)\left(1-\Phi\left(\tfrac{(\mu_p(o)-\mu_p(x))}{\sqrt{K_p(x,x)+2\sigma^2+K_p(o,o)-2K_p(x,o)}}\right)\right)\\
& + \tfrac{K_p(x,x)-K_p(x,o)}{\sqrt {K_p(x,x)+2\sigma^2+K_p(o,o)-2K_p(x,o)}}\\
&~~\cdot \phi\left(\tfrac{\mu_p(o)-\mu_p(x)}{\sqrt{K_p(x,x)+2\sigma^2+K_p(o,o)-2K_p(x,o)}}\right)
\end{aligned}
\end{equation}

\end{lemma}
\begin{proof}
We will compute $E[\nu(x)I_{\{\nu(x)>\nu(o)+n(o)-n(x)\}}]$ in two steps. First, we assume that $\nu(o),n(o)-n(x)$ are given and, therefore, we condition the joint PDF of $\nu(x),n(x),\nu(o),n(o)$ on $\nu(o),n(o)$. Since only the variables $\nu(x),\nu(o)$ are dependent, then we have
\begin{equation}
\begin{aligned}
&p(\nu(x)|\nu(o))={\scriptstyle N\left(\nu(x);\mu_p(x)+\tfrac{K_p(x,o)}{K_p(o,o)}(\nu(o)-\mu_p(o)),
K_p(x,x)-\tfrac{K_p^2(x,o)}{K_p(o,o)}\right)}.
\end{aligned}
\end{equation}

Therefore, we can apply \eqref{eq:mills} conditionally on  $\nu(o),n(o)-n(x)$ which leads to
\begin{equation}
\label{eq:millscond}
\begin{aligned}
&E[\nu(x)I_{\{\nu(x)>\nu(o)+n(o)-n(x)\}}|\nu(o),n(o),n(x)]\\
&=m_1 \left(1-\Phi\left(\tfrac{\nu(o)+n(o)-n(x)-m_1}{\sigma_1}\right)\right)+\sigma_1 \phi\left(\tfrac{\nu(o)+n(o)-n(x)-m_1}{\sigma_1}\right)
\end{aligned}
\end{equation}
Now observe that
\begin{equation}
\label{eq:intm1}
\begin{aligned}
E[m_1]&=\int \left(\mu_p(x)+\tfrac{K_p(x,o)}{K_p(o,o)}(\nu(o)-\mu_p(o))\right)\\
&N(\nu(o);\mu_p(o),K_p(o,o))d\nu(o)
d\nu(o)=\mu_p(x).
\end{aligned}
\end{equation}
and
\begin{equation}
\label{eq:intm1}
\begin{aligned}
&E\left[m_1\Phi\left(\tfrac{\nu(o)+n(o)-n(x)-m_1}{\sigma_1}\right)\right]\\
&=E\left[\left(\mu_p(x)+\tfrac{K_p(x,o)}{K_p(o,o)}(\nu(o)-\mu_p(o))\right)\Phi\left(\tfrac{\nu(o)+n(o)-n(x)-m_1}{\sigma_1}\right)\right]\\
&=\left(\mu_p(x)-\tfrac{K_p(x,o)}{K_p(o,o)}\mu_p(o)\right) E\left[\Phi\left(\tfrac{\nu(o)+n(o)-n(x)-m_1}{\sigma_1}\right)\right]\\
&+\tfrac{K_p(x,o)}{K_p(o,o)}E\left[ \nu(o)\Phi\left(\tfrac{\nu(o)+n(o)-n(x)-m_1}{\sigma_1}\right)\right]
\end{aligned}
\end{equation}
The expectations are with respect to $\nu(o),n(o),n(x)$.
Now we use \eqref{eq:Resprobitint} to get the following result:
\begin{equation}
\label{eq:probitint1}
\begin{aligned}
\int \Phi\left(\tfrac{\nu(o)+n(o)-n(x)-m_1}{\sigma_1}\right)N(n(o)-n(x);0,2\sigma^2)dn(o)
&=\Phi\left(\tfrac{\nu(o)-m_1}{\sqrt{\sigma_1^2+2\sigma^2}}\right),
\end{aligned}
\end{equation}
and so:
\begin{equation}
\label{eq:probitint2}
\begin{aligned}
E\left[\Phi\left(\tfrac{\nu(o)+n(o)-n(x)-m_1}{\sigma_1}\right)\right]&=\int \Phi\left(\tfrac{\nu(o)-m_1}{\sqrt{\sigma_1^2+2\sigma^2}}\right)N(\nu(o);\mu_p(o),K_p(o,o))d\nu(o)\\
&=\int \Phi\left(\tfrac{\nu(o)\tfrac{K_p(o,o)-K_p(x,o)}{K_p(o,o)}+m_2}{\sqrt{\sigma_1^2+2\sigma^2}}\right)\\
&~~~~~~~~N(\nu(o);\mu_p(o),K_p(o,o))d\nu(o)\\
&=\int \Phi\left(\tfrac{z\tfrac{K_p(o,o)-K_p(x,o)}{\sqrt{K_p(o,o)}}+m_2+\tfrac{K_p(o,o)-K_p(x,o)}{K_p(o,o)}\mu_p(o)}{\sqrt{\sigma_1^2+2\sigma^2}}\right)\\
&~~~~~~~~N(z;0,1)dz\\
&=\Phi\left(\tfrac{\mu_p(o)\tfrac{K_p(o,o)-K_p(x,o)}{\sqrt{K_p(o,o)}}+m_2\sqrt{K_p(o,o)}}{\sqrt{K_p(o,o)(\sigma_1^2+2\sigma^2)+(K_p(o,o)-K_p(x,o))^2}}\right)\\
&=\Phi\left(\tfrac{\sqrt{K_p(o,o)}(\mu_p(o)-\mu_p(x))}{\sqrt{K_p(o,o)(\sigma_1^2+2\sigma^2)+(K_p(o,o)-K_p(x,o))^2}}\right),\\
\end{aligned}
\end{equation}
with $m_2=\tfrac{-K_p(o,o)\mu_p(x)+K_p(x,o)\mu_p(o)}{
K_p(o,o)}$. Similarly, we have that
\begin{equation}
\label{eq:probitint2bis}
\begin{aligned}
&E\left[\nu(o)\Phi\left(\tfrac{\nu(o)+n(o)-n(x)-m_1}{\sigma_1}\right)\right]\\
&=\int \nu(o) \Phi\left(\tfrac{\nu(o)-m_1}{\sqrt{\sigma_1^2+2\sigma^2}}\right)N(\nu(o);\mu_p(o),K_p(o,o))d\nu(o)\\
&=\int \nu(o) \Phi\left(\tfrac{\nu(o)\tfrac{K_p(o,o)-K_p(x,o)}{K_p(o,o)}+m_2}{\sqrt{\sigma_1^2+2\sigma^2}}\right)\\
&~~~~~~~~N(\nu(o);\mu_p(o),K_p(o,o))d\nu(o)\\
&=\int \Phi\left(\tfrac{z\tfrac{K_p(o,o)-K_p(x,o)}{\sqrt{K_p(o,o)}}+m_2+\tfrac{K_p(o,o)-K_p(x,o)}{K_p(o,o)}\mu_p(o)}{\sqrt{\sigma_1^2+2\sigma^2}}\right)\\
& ~~~~\left(z \sqrt{K_p(o,o)}+\mu_p(o)\right)N(z;0,1)dz\\
\end{aligned}
\end{equation}
We separate the sum:
\begin{equation}
\label{eq:probitint2bis1}
\begin{aligned}
&\int \Phi\left(\tfrac{z\tfrac{K_p(o,o)-K_p(x,o)}{\sqrt{K_p(o,o)}}+m_2+\tfrac{K_p(o,o)-K_p(x,o)}{K_p(o,o)}\mu_p(o)}{\sqrt{\sigma_1^2+2\sigma^2}}\right)\\
& ~~~~\mu_p(o) N(z;0,1)dz\\
&=\mu_p(o)\Phi\left(\tfrac{\mu_p(o)\tfrac{K_p(o,o)-K_p(x,o)}{\sqrt{K_p(o,o)}}+m_2\sqrt{K_p(o,o)}}{\sqrt{K_p(o,o)(\sigma_1^2+2\sigma^2)+(K_p(o,o)-K_p(x,o))^2}}\right)\\
&=\mu_p(o)\Phi\left(\tfrac{\sqrt{K_p(o,o)}(\mu_p(o)-\mu_p(x))}{\sqrt{K_p(o,o)(\sigma_1^2+2\sigma^2)+(K_p(o,o)-K_p(x,o))^2}}\right).\\
\end{aligned}
\end{equation}
The other term in the sum
\begin{equation}
\label{eq:probitint2bis2}
\begin{aligned}
&\int \Phi\left(\tfrac{z\tfrac{K_p(o,o)-K_p(x,o)}{\sqrt{K_p(o,o)}}+m_2+\tfrac{K_p(o,o)-K_p(x,o)}{K_p(o,o)}\mu_p(o)}{\sqrt{\sigma_1^2+2\sigma^2}}\right)\\
& ~~~~z \sqrt{K_p(o,o)} N(z;0,1)dz\\
&=\tfrac{\sqrt{K_p(o,o)}(K_p(o,o)-K_p(x,o))}{\sqrt{K_p(o,o)(\sigma_1^2+\sigma^2)+(K_p(o,o)-K_p(x,o))^2}}\\
&\phi\left(\tfrac{\mu_p(o)\tfrac{K_p(o,o)-K_p(x,o)}{\sqrt{K_p(o,o)}}+m_2\sqrt{K_p(o,o)}}{\sqrt{K_p(o,o)(\sigma_1^2+2\sigma^2)+(K_p(o,o)-K_p(x,o))^2}}\right)\\
&=\tfrac{\sqrt{K_p(o,o)}(K_p(o,o)-K_p(x,o))}{\sqrt{K_p(o,o)(\sigma_1^2+2\sigma^2)+(K_p(o,o)-K_p(x,o))^2}}\\
&\phi\left(\tfrac{\sqrt{K_p(o,o)}(\mu_p(o)-\mu_p(x))}{\sqrt{K_p(o,o)(\sigma_1^2+2\sigma^2)+(K_p(o,o)-K_p(x,o))^2}}\right).\\
\end{aligned}
\end{equation}
where we have used \eqref{eq:xResprobitint}. 
Finally, we consider
\begin{equation}
\label{eq:probitint3pre}
\begin{aligned}
&\int \phi\left(\tfrac{\nu(o)+n(o)-n(x)-m_1}{\sigma_1}\right)N(n(o)-n(x);0,2\sigma^2)dn(o)\\
&=\tfrac{\sigma_1}{\sqrt{\sigma_1^2+2\sigma^2}}\phi\left(\tfrac{\nu(o)-m_1}{\sqrt{\sigma_1^2+2\sigma^2}}\right),\\
\end{aligned}
\end{equation}
where the last equality follows by \eqref{eq:ResGPDFint}. We use \eqref{eq:probitint2} to get:
\begin{equation}
\label{eq:probitint3} 
\begin{aligned}
&\int \tfrac{\sigma_1}{\sqrt{\sigma_1^2+2\sigma^2}}\phi\left(\tfrac{\nu(o)-m_1}{\sqrt{\sigma_1^2+2\sigma^2}}\right)N(\nu(o);\mu_p(o),K_p(o,o))d\nu(o)\\
&=\int \tfrac{\sigma_1}{\sqrt{\sigma_1^2+2\sigma^2}}\phi\left(\tfrac{\nu(o)\tfrac{K_p(o,o)-K_p(x,o)}{K_p(o,o)}+m_2}{\sqrt{\sigma_1^2+2\sigma^2}}\right)\\
&N(\nu(o);\mu_p(o),K_p(o,o))d\nu(o)\\
&=\int \tfrac{\sigma_1}{\sqrt{\sigma_1^2+2\sigma^2}}\phi\left(\tfrac{z\tfrac{K_p(o,o)-K_p(x,o)}{\sqrt{K_p(o,o)}}+m_2+\tfrac{K_p(o,o)-K_p(x,o)}{K_p(o,o)}\mu_p(o)}{\sqrt{\sigma_1^2+2\sigma^2}}\right)\\
&N(z;0,1)d\nu(o)\\
&=\tfrac{\sqrt{K_p(o,o)}\sigma_1}{\sqrt{K_p(o,o)(\sigma_1^2+2\sigma^2)+(K_p(o,o)-K_p(x,o))^2}}\\
&\phi\left(\tfrac{\sqrt{K_p(o,o)}(\mu_p(o)-\mu_p(x))}{\sqrt{K_p(o,o)(\sigma_1^2+2\sigma^2)+(K_p(o,o)-K_p(x,o))^2}}\right).\\
\end{aligned}
\end{equation}

Therefore, from \eqref{eq:millscond} and \eqref{eq:probitint2}--\eqref{eq:probitint3}, we obtain
\begin{equation}
\label{eq:millsthird}
\begin{aligned}
&E[(\nu(x)+n(x))I_{\{\nu(x)+n(x)>\nu(o)+n(o)\}}]=\mu_p(x)\\
&-\left(\mu_p(x)-\tfrac{K_p(x,o)}{K_p(o,o)}\mu_p(o)\right)\\
&\cdot \Phi\left(\tfrac{\sqrt{K_p(o,o)}(\mu_p(o)-\mu_p(x))}{\sqrt{K_p(o,o)(\sigma_1^2+2\sigma^2)+(K_p(o,o)-K_p(x,o))^2}}\right)\\
& - \tfrac{K_p(x,o)}{K_p(o,o)} \mu_p(o)\Phi\left(\tfrac{\sqrt{K_p(o,o)}(\mu_p(o)-\mu_p(x))}{\sqrt{K_p(o,o)(\sigma_1^2+2\sigma^2)+(K_p(o,o)-K_p(x,o))^2}}\right)\\
& - \tfrac{K_p(x,o)}{K_p(o,o)} \tfrac{\sqrt{K_p(o,o)}(K_p(o,o)-K_p(x,o))}{\sqrt{K_p(o,o)(\sigma_1^2+2\sigma^2)+(K_p(o,o)-K_p(x,o))^2}}\\
&\phi\left(\tfrac{\sqrt{K_p(o,o)}(\mu_p(o)-\mu_p(x))}{\sqrt{K_p(o,o)(\sigma_1^2+2\sigma^2)+(K_p(o,o)-K_p(x,o))^2}}\right)\\
&+\tfrac{\sqrt{K_p(o,o)}\sigma^2_1}{\sqrt{K_p(o,o)(\sigma_1^2+2\sigma^2)+(K_p(o,o)-K_p(x,o))^2}}\\
&\phi\left(\tfrac{\sqrt{K_p(o,o)}(\mu_p(o)-\mu_p(x))}{\sqrt{K_p(o,o)(\sigma_1^2+2\sigma^2)+(K_p(o,o)-K_p(x,o))^2}}\right)\\
&=\mu_p(x)\left(1-\Phi\left(\tfrac{\sqrt{K_p(o,o)}(\mu_p(o)-\mu_p(x))}{\sqrt{K_p(o,o)(\sigma_1^2+2\sigma^2)+(K_p(o,o)-K_p(x,o))^2}}\right)\right)\\
& + \tfrac{\sqrt{K_p(o,o)}(K_p(x,x)-K_p(x,o))}{\sqrt{K_p(o,o)(\sigma_1^2+2\sigma^2)+(K_p(o,o)-K_p(x,o))^2}}\\
&\phi\left(\tfrac{\sqrt{K_p(o,o)}(\mu_p(o)-\mu_p(x))}{\sqrt{K_p(o,o)(\sigma_1^2+2\sigma^2)+(K_p(o,o)-K_p(x,o))^2}}\right)\\
\end{aligned}
\end{equation}

Note that
\begin{equation}
\label{eq:var00}
\begin{aligned}
&K_p(o,o)(\sigma_1^2+2\sigma^2)+(K_p(o,o)-K_p(x,o))^2\\
&=K_p(o,o)\left(K_p(x,x)-\tfrac{K_p^2(x,o)}{K_p(o,o)}+2\sigma^2\right)\\
&+(K_p(o,o)-K_p(x,o))^2\\
&=K_p(o,o)K_p(x,x)-K_p^2(x,o)+2\sigma^2K_p(o,o)\\
&+K_p^2(o,o)+K_p^2(x,o)-2K_p(x,o)K_p(o,o)\\
&=K_p(o,o)(K_p(x,x)+2\sigma^2+K_p(o,o)-2K_p(x,o)).\\
\end{aligned}
\end{equation}
Therefore, we have that
\begin{equation}
\label{eq:millsthird11}
\begin{aligned}
E[\nu(x)I_{\{\nu(x)+n(x)>\nu(o)+n(o)\}}]&=\mu_p(x)\left(1-\Phi\left(\tfrac{(\mu_p(o)-\mu_p(x))}{\sqrt{K_p(x,x)+2\sigma^2+K_p(o,o)-2K_p(x,o)}}\right)\right)\\
& + \tfrac{K_p(x,x)-K_p(x,o)}{\sqrt {K_p(x,x)+2\sigma^2+K_p(o,o)-2K_p(x,o)}}\\
&\phi\left(\tfrac{\mu_p(o)-\mu_p(x)}{\sqrt{K_p(x,x)+2\sigma^2+K_p(o,o)-2K_p(x,o)}}\right)
\end{aligned}
\end{equation}
    \end{proof}

\section{Proofs}
\label{app:proof}
We now move on to the main results.

\paragraph{Proof of Lemma \ref{lem:1}}
The expected value for $DEF$  follows from Lemma \ref{lem:preference_learn} by summing $E[\nu(x)I_{\{\nu(x)+n(x)>\nu(o)+n(o)\}}]$ and $E[\nu(o)I_{\{\nu(x)+n(x)<\nu(o)+n(o)\}}]$. The expected values for $IMM,DoN$ are straightforward.

\paragraph{Proof of Proposition \ref{prop:1}}
The results follow from Lemma \ref{lem:1}  by considering whether or not the limits $K_0(x,x),K_0(o,o),K_0(o,x) \rightarrow 0$ and $\sigma \rightarrow 0$ are taken. We make the assumption that whenever $K_o(x,x),K_o(o,o),K_o(o,x) \rightarrow 0$  it implies that $K_p(x,x),K_p(o,o),K_p(o,x) \rightarrow 0$ (a-priori we have a Dirac's delta). Since $K_p$ depends on both $K_0$ and $\sigma$, we always take the limit with respect to $\sigma$ first.

 If $S$ is \textbf{rational} and $R$ has \textbf{no uncertainty}, then the expected payoffs can be computed from \eqref{eq:DEFexp}, \eqref{eq:IMMexp} and \eqref{eq:OFFexp}.
  The values are
 $E[DEF]=\max(\mu_p(x),\mu_p(o))$
  $E[IMM]=\mu_p(x)$ and $E[DoN]=\mu_p(o)$. Therefore, DEF is never dominated.
  
   If $S$ is \textbf{bounded-rational} and $R$ has \textbf{no uncertainty}, then $\sigma>0$ and the payoffs are:
   $E[DEF]=p\mu_p(x)+(1-p)\mu_p(o)$, 
  $E[IMM]=\mu_p(x)$ and $E[DoN]=\mu_p(o)$, where
  $p=\Phi\left(\tfrac{\mu_p(x)-\mu_p(o)}{\sqrt{2\sigma^2}}\right)$
   Therefore, $p \in (0,1)$ and   
   DEF is never optimal.

If $S$ is \textbf{rational} and $R$ has \textbf{uncertainty}, then 
$E[DEF]=p\mu_p(x)+(1-p)\mu_p(o)+e$,   $E[IMM]=\mu_p(x)$ and $E[DoN]=\mu_p(o)$, where  $p=\Phi\left(\tfrac{\mu_p(x)-\mu_p(o)}{\sqrt{K_p(o,o)+K_p(x,x)-2K_p(x,o)}}\right)$ and 
$$
 \resizebox{0.91\hsize}{!}{$
\begin{aligned}
e&= \tfrac{K_p(x,x)-K_p(x,o)}{\sqrt {K_p(x,x)+K_p(o,o)-2K_p(x,o)}}
\phi\left(\tfrac{\mu_p(o)-\mu_p(x)}{\sqrt{K_p(x,x)+K_p(o,o)-2K_p(x,o)}}\right)\\
& + \tfrac{K_p(o,o)-K_p(x,o)}{\sqrt {K_p(x,x)+K_p(o,o)-2K_p(x,o)}}\phi\left(\tfrac{\mu_p(x)-\mu_p(o)}{\sqrt{K_p(x,x)+K_p(o,o)-2K_p(x,o)}}\right).
\end{aligned}$}
$$
When $\bH$ is rational, then 
\begin{equation}
\label{eq:payoffprop1Rproof}
\begin{aligned}
&u_R(t,DEF,b^*(DEF))=\nu(o)I_{\{\nu(o)>\nu(x)\}}+\nu(x)I_{\{\nu(x)>\nu(o)\}}=\max(\nu(o),\nu(x))
\end{aligned}
\end{equation}
$\max$ is a convex function and, therefore, by Jensen's inequality 
$E[\max(\nu(o),\nu(x))]\geq \max(E[\nu(o)],E[\nu(x)])$. Therefore, $p\mu_p(x)+(1-p)\mu_p(o)+e\geq \max(\mu_p(x),\mu_p(o))$ and DEF is always optimal.

The last case follows directly from Lemma \ref{lem:1}.

\paragraph{Proof of Lemma \ref{new:lemma}}
The  results follow by Lemma \ref{lem:probitint} and \ref{lem:preference_learn}.

\paragraph{Proof of Proposition \ref{prop:acq}}
The first statement follows by Lemma \ref{new:lemma} as under rationality and no-uncertainty, we have that:
 \begin{align}
  &\text{natural:} && \max(\nu(x),\nu(o)),\\
  &  \text{corporate:} && I_{\{\nu(x)>\nu(o)\}},\\
   &   \text{collaborative:} && \nu(o)I_{\{\nu(o)>\nu(x)\}}+\nu(x)I_{\{\nu(x)>\nu(o)\}},
 \end{align}
 which $\bR$ will maximise with respect to $x$.

 Under no-uncertainty and bounded-rationality, we have that:

  \begin{align}
  &\text{natural:} && \max(\nu(x),\nu(o)),\\
  &  \text{corporate:} && \Phi\left(\frac{\nu(x)-\nu(o)}{\sqrt{2}\sigma}\right)\\
   &   \text{collaborative:} && \nu(o)\Phi\left(\frac{\nu(o)-\nu(x)}{\sqrt{2}\sigma}\right)+\nu(x)\Phi\left(\frac{\nu(x)-\nu(o)}{\sqrt{2}\sigma}\right).
 \end{align}
 Since the Gaussian CDF is monotonic increasing, we see that for all the strategies the maximum will be achieved for $x=x^*$.  

\paragraph{Proof of Corollary \ref{co:2}}
The results follow from Proposition \ref{prop:1} (and \eqref{eq:absx}) after subtracting $\beta$ to the expected payoff for DEF, where
\begin{align}
\nonumber
\beta &=\gamma E[|\nu(o)|] =\gamma \mu_p(o) \left(1-2\Phi\left(\tfrac{-\mu_p(o)}{\sqrt{K_p(o,o)}}\right)\right)+2\gamma\sqrt{K_p(o,o)} \phi\left(\tfrac{-\mu_p(o)}{\sqrt{K_p(o,o)}}\right).
\end{align}
 If $S$ is \textbf{rational} and $R$ has \textbf{no uncertainty}, then the expected payoffs can be computed from \eqref{eq:DEFexp}, \eqref{eq:IMMexp} and \eqref{eq:OFFexp}.
  The values are
 $E[DEF]=\max(\mu_p(x),\mu_p(o))-\gamma|\mu_p(o)|$
  $E[IMM]=\mu_p(x)$ and $E[DoN]=\mu_p(o)$. Therefore, DEF is never optimal.

   If $S$ is \textbf{bounded-rational} and $R$ has \textbf{no uncertainty}, then $\sigma>0$ and the payoffs are:
   $E[DEF]=p\mu_p(x)+(1-p)\mu_p(o)-\gamma|\mu_p(o)|$
  $E[IMM]=\mu_p(x)$ and $E[DoN]=\mu_p(o)$, where
  $p=\Phi\left(\tfrac{\mu_p(x)-\mu_p(o)}{\sqrt{2\sigma^2}}\right)$
   Therefore, $p \in (0,1)$ and   
   DEF is never optimal.

   If $S$ is \textbf{rational} and $R$ has \textbf{uncertainty}, then 
$E[DEF]=p\mu_p(x)+(1-p)\mu_p(o)+e-\gamma'|\mu_p(o)|$
  $E[IMM]=\mu_p(x)$ and $E[DoN]=\mu_p(o)$. Therefore, 
  DEF is optimal if 
$p\mu_p(x) + (1-p)\mu_p(o)+e-\gamma|\mu_p(o)|\geq \max\left(\mu_p(x) ,\mu_p(o)\right)$.
The last case follows similarly from Proposition \ref{prop:1}

\paragraph{Proof of Lemma \ref{lem:2}}
The expected value for $DEF$  is equal to the sum of $E[\nu(x)I_{\{\nu(x)>\nu(o)+\sigma)\}}]$, $E[\nu(o)I_{\{\nu(o)>\nu(x)+\sigma)\}}]$ and \\ $\{E[\nu(x)I_{\{|\nu(x)-\nu(o)|\leq \sigma)\}}],E[\nu(o)I_{\{|\nu(x)-\nu(o)|\leq \sigma)\}}]\}$.
First we have that
 \begin{equation}
\begin{aligned}
p(\nu(x)|\nu(o))=N(\nu(x);m_1,\sigma_1^2)=
&{\scriptstyle N\left(\nu(x);\mu_p(x)+\tfrac{K_p(x,o)}{K_p(o,o)}(\nu(o)-\mu_p(o)),
K_p(x,x)-\tfrac{K_p^2(x,o)}{K_p(o,o)}\right)}.
\end{aligned}
\end{equation}
Therefore, we can apply \eqref{eq:mills} conditionally on  $\nu(o)$ which leads to
\begin{equation}
\label{eq:millscond42}
\begin{aligned}
&E[\nu(x)I_{\{\nu(x)>\nu(o)+\sigma\}}|\nu(o)]\\
&=m_1 \left(1-\Phi\left(\tfrac{\nu(o)+\sigma-m_1}{\sigma_1}\right)\right)+\sigma_1 \phi\left(\tfrac{\nu(o)+\sigma-m_1}{\sigma_1}\right)
\end{aligned}
\end{equation}
Now observe that
\begin{equation}
\label{eq:intm142}
\begin{aligned}
E[m_1]&=\int \left(\mu_p(x)+\tfrac{K_p(x,o)}{K_p(o,o)}(\nu(o)-\mu_p(o))\right)\\
&N(\nu(o);\mu_p(o),K_p(o,o))d\nu(o)
d\nu(o)=\mu_p(x),
\end{aligned}
\end{equation}
and
\begin{equation}
\label{eq:intm1421}
\begin{aligned}
E\left[m_1\Phi\left(\tfrac{\nu(o)+\sigma-m_1}{\sigma_1}\right)\right]
&=E\left[\left(\mu_p(x)+\tfrac{K_p(x,o)}{K_p(o,o)}(\nu(o)-\mu_p(o))\right)\Phi\left(\tfrac{\nu(o)+\sigma-m_1}{\sigma_1}\right)\right]\\
&=\left(\mu_p(x)-\tfrac{K_p(x,o)}{K_p(o,o)}\mu_p(o)\right) E\left[\Phi\left(\tfrac{\nu(o)+\sigma-m_1}{\sigma_1}\right)\right]\\
&+\tfrac{K_p(x,o)}{K_p(o,o)}E\left[ \nu(o)\Phi\left(\tfrac{\nu(o)+\sigma-m_1}{\sigma_1}\right)\right]
\end{aligned}
\end{equation}
The expectations are with respect to $\nu(o)$.
Now we use \eqref{eq:Resprobitint} to get 

\begin{equation}
\label{eq:probitint242}
\begin{aligned}
E\left[\Phi\left(\tfrac{\nu(o)+\sigma-m_1}{\sigma_1}\right)\right]
&=\int \Phi\left(\tfrac{\nu(o)+\sigma-m_1}{\sigma_1}\right)N(\nu(o);\mu_p(o),K_p(o,o))d\nu(o)\\
&=\int \Phi\left(\tfrac{\nu(o)\tfrac{K_p(o,o)-K_p(x,o)}{K_p(o,o)}+m_2}{\sigma_1}\right)\\
&~~~~~~~~N(\nu(o);\mu_p(o),K_p(o,o))d\nu(o)\\
&=\int \Phi\left(\tfrac{z\tfrac{K_p(o,o)-K_p(x,o)}{\sqrt{K_p(o,o)}}+m_2+\tfrac{K_p(o,o)-K_p(x,o)}{K_p(o,o)}\mu_p(o)}{\sigma_1}\right)\\
&~~~~~~~~N(z;0,1)dz\\
&=\Phi\left(\tfrac{\mu_p(o)\tfrac{K_p(o,o)-K_p(x,o)}{\sqrt{K_p(o,o)}}+m_2\sqrt{K_p(o,o)}}{\sqrt{K_p(o,o)\sigma_1^2+(K_p(o,o)-K_p(x,o))^2}}\right)\\
&=\Phi\left(\tfrac{\sqrt{K_p(o,o)}(\mu_p(o)+\sigma-\mu_p(x))}{\sqrt{K_p(o,o)\sigma_1^2+(K_p(o,o)-K_p(x,o))^2}}\right),\\
\end{aligned}
\end{equation}
with $m_2=\tfrac{K_p(o,o)(\sigma-\mu_p(x))+K_p(x,o)\mu_p(o)}{
K_p(o,o)}$. Similarly, we have that
\begin{equation}
\label{eq:probitint2bis42}
\begin{aligned}
E\left[\nu(o)\Phi\left(\tfrac{\nu(o)+\sigma-m_1}{\sigma_1}\right)\right]
&=\int \nu(o) \Phi\left(\tfrac{\nu(o)+\sigma-m_1}{\sigma_1}\right) N(\nu(o);\mu_p(o),K_p(o,o))d\nu(o)\\
&=\int \nu(o) \Phi\left(\tfrac{\nu(o)\tfrac{K_p(o,o)-K_p(x,o)}{K_p(o,o)}+m_2}{\sigma_1}\right)\\
&~~~~~~~~N(\nu(o);\mu_p(o),K_p(o,o))d\nu(o)\\
&=\int \Phi\left(\tfrac{z\tfrac{K_p(o,o)-K_p(x,o)}{\sqrt{K_p(o,o)}}+m_2+\tfrac{K_p(o,o)-K_p(x,o)}{K_p(o,o)}\mu_p(o)}{\sigma_1}\right)\\
& ~~~~\left(z \sqrt{K_p(o,o)}+\mu_p(o)\right)N(z;0,1)dz\\
\end{aligned}
\end{equation}
We separate the sum:
\begin{equation}
\label{eq:probitint2bis142}
\begin{aligned}
&\int \Phi\left(\tfrac{z\tfrac{K_p(o,o)-K_p(x,o)}{\sqrt{K_p(o,o)}}+m_2+\tfrac{K_p(o,o)-K_p(x,o)}{K_p(o,o)}\mu_p(o)}{\sigma_1}\right)\\
& ~~~~\mu_p(o) N(z;0,1)dz\\
&=\mu_p(o)\Phi\left(\tfrac{\mu_p(o)\tfrac{K_p(o,o)-K_p(x,o)}{\sqrt{K_p(o,o)}}+m_2\sqrt{K_p(o,o)}}{\sqrt{K_p(o,o)\sigma_1^2+(K_p(o,o)-K_p(x,o))^2}}\right)\\
&=\mu_p(o)\Phi\left(\tfrac{\sqrt{K_p(o,o)}(\mu_p(o)+\sigma-\mu_p(x))}{\sqrt{K_p(o,o)\sigma_1^2+(K_p(o,o)-K_p(x,o))^2}}\right).\\
\end{aligned}
\end{equation}
The other term in the sum
\begin{equation}
\label{eq:probitint2bis242}
\begin{aligned}
&\int \Phi\left(\tfrac{z\tfrac{K_p(o,o)-K_p(x,o)}{\sqrt{K_p(o,o)}}+m_2+\tfrac{K_p(o,o)-K_p(x,o)}{K_p(o,o)}\mu_p(o)}{\sigma_1}\right)\\
& ~~~~z \sqrt{K_p(o,o)} N(z;0,1)dz\\
&=\tfrac{\sqrt{K_p(o,o)}(K_p(o,o)-K_p(x,o))}{\sqrt{K_p(o,o)\sigma_1^2+(K_p(o,o)-K_p(x,o))^2}}\\
&\phi\left(\tfrac{\mu_p(o)\tfrac{K_p(o,o)-K_p(x,o)}{\sqrt{K_p(o,o)}}+m_2\sqrt{K_p(o,o)}}{\sqrt{K_p(o,o)(\sigma_1^2+2\sigma^2)+(K_p(o,o)-K_p(x,o))^2}}\right)\\
&=\tfrac{\sqrt{K_p(o,o)}(K_p(o,o)-K_p(x,o))}{\sqrt{K_p(o,o)\sigma_1^2+(K_p(o,o)-K_p(x,o))^2}}\\
&\phi\left(\tfrac{\sqrt{K_p(o,o)}(\mu_p(o)+\sigma-\mu_p(x))}{\sqrt{K_p(o,o)\sigma_1^2+(K_p(o,o)-K_p(x,o))^2}}\right).\\
\end{aligned}
\end{equation}
where we have used \eqref{eq:xResprobitint}. 
Finally, we consider
\begin{equation}
\label{eq:probitint342} 
\begin{aligned}
&\int \phi\left(\tfrac{\nu(o)+\sigma-m_1}{\sigma_1}\right)N(\nu(o);\mu_p(o),K_p(o,o))d\nu(o)\\
&=\int \phi\left(\tfrac{z\tfrac{K_p(o,o)-K_p(x,o)}{\sqrt{K_p(o,o)}}+m_2+\tfrac{K_p(o,o)-K_p(x,o)}{K_p(o,o)}\mu_p(o)}{\sigma_1}\right)\\
&N(z;0,1)d\nu(o)\\
&=\tfrac{\sqrt{K_p(o,o)}\sigma_1}{\sqrt{K_p(o,o)_p\sigma_1^2+(K_p(o,o)-K_p(x,o))^2}}\\
&\phi\left(\tfrac{\sqrt{K_p(o,o)}(\mu_p(o)+\sigma-\mu_p(x))}{\sqrt{K_p(o,o)\sigma_1^2+(K_p(o,o)-K_p(x,o))^2}}\right).\\
\end{aligned}
\end{equation}

Therefore, from \eqref{eq:millscond} and \eqref{eq:probitint2}--\eqref{eq:probitint3}, we obtain
\begin{equation}
\label{eq:millsthird42}
\begin{aligned}
&E[\nu(x)I_{\{\nu(x)>\nu(o)+\sigma\}}]=\mu_p(x)\\
&-\left(\mu_p(x)-\tfrac{K_p(x,o)}{K_p(o,o)}\mu_p(o)\right)\\
&\cdot \Phi\left(\tfrac{\sqrt{K_p(o,o)}(\mu_p(o)+\sigma-\mu_p(x))}{\sqrt{K_p(o,o)\sigma_1^2+(K_p(o,o)-K_p(x,o))^2}}\right)\\
& - \tfrac{K_p(x,o)}{K_p(o,o)} \mu_p(o)\Phi\left(\tfrac{\sqrt{K_p(o,o)}(\mu_p(o)+\sigma-\mu_p(x))}{\sqrt{K_p(o,o)\sigma_1^2+(K_p(o,o)-K_p(x,o))^2}}\right)\\
& - \tfrac{K_p(x,o)}{K_p(o,o)} \tfrac{\sqrt{K_p(o,o)}(K_p(o,o)-K_p(x,o))}{\sqrt{K_p(o,o)\sigma_1^2+(K_p(o,o)-K_p(x,o))^2}}\\
&\phi\left(\tfrac{\sqrt{K_p(o,o)}(\mu_p(o)+\sigma-\mu_p(x))}{\sqrt{K_p(o,o)\sigma_1^2+(K_p(o,o)-K_p(x,o))^2}}\right)\\
&+\tfrac{\sqrt{K_p(o,o)}\sigma^2_1}{\sqrt{K_p(o,o)\sigma_1^2+(K_p(o,o)-K_p(x,o))^2}}\\
&\phi\left(\tfrac{\sqrt{K_p(o,o)}(\mu_p(o)+\sigma-\mu_p(x))}{\sqrt{K_p(o,o)\sigma_1^2+(K_p(o,o)-K_p(x,o))^2}}\right)\\
&=\mu_p(x)\left(1-\Phi\left(\tfrac{\sqrt{K_p(o,o)}(\mu_p(o)+\sigma-\mu_p(x))}{\sqrt{K_p(o,o)\sigma_1^2+(K_p(o,o)-K_p(x,o))^2}}\right)\right)\\
& + \tfrac{\sqrt{K_p(o,o)}(K_p(x,x)-K_p(x,o))}{\sqrt{K_p(o,o)(\sigma_1^2+2\sigma^2)+(K_p(o,o)-K_p(x,o))^2}}\\
&\phi\left(\tfrac{\sqrt{K_p(o,o)}(\mu_p(o)+\sigma-\mu_p(x))}{\sqrt{K_p(o,o)\sigma_1^2+(K_p(o,o)-K_p(x,o))^2}}\right)\\
\end{aligned}
\end{equation}

Note that
\begin{equation}
\label{eq:var0042}
\begin{aligned}
&K_p(o,o)\sigma_1^2+(K_p(o,o)-K_p(x,o))^2\\
&=K_p(o,o)\left(K_p(x,x)-\tfrac{K_p^2(x,o)}{K_p(o,o)}\right)+(K_p(o,o)-K_p(x,o))^2\\
&=K_p(o,o)K_p(x,x)-K_p^2(x,o)\\
&+K_p^2(o,o)+K_p^2(x,o)-2K_p(x,o)K_p(o,o)\\
&=K_p(o,o)(K_p(x,x)+K_p(o,o)-2K_p(x,o)).\\
\end{aligned}
\end{equation}
Therefore, we have that
\begin{equation}
\label{eq:millsthird1142}
\begin{aligned}
&E[\nu(x)I_{\{\nu(x)>\nu(o)+\sigma\}}]=\\
&=\mu_p(x)\left(1-\Phi\left(\tfrac{(\mu_p(o)+\sigma-\mu_p(x))}{\sqrt{K_p(x,x)+K_p(o,o)-2K_p(x,o)}}\right)\right)\\
& + \tfrac{K_p(x,x)-K_p(x,o)}{\sqrt {K_p(x,x)+K_p(o,o)-2K_p(x,o)}}\\
&\phi\left(\tfrac{\mu_p(o)+\sigma-\mu_p(x)}{\sqrt{K_p(x,x)+K_p(o,o)-2K_p(x,o)}}\right)
\end{aligned}
\end{equation}
and 
\begin{equation}
\label{eq:millsthird1142sec}
\begin{aligned}
&E[\nu(o)I_{\{\nu(o)>\nu(x)+\sigma\}}]=\\
&=\mu_p(o)\left(1-\Phi\left(\tfrac{(\mu_p(x)+\sigma-\mu_p(o))}{\sqrt{K_p(x,x)+K_p(o,o)-2K_p(x,o)}}\right)\right)\\
& + \tfrac{K_p(o,o)-K_p(x,o)}{\sqrt {K_p(x,x)+K_p(o,o)-2K_p(x,o)}}\\
&\phi\left(\tfrac{\mu_p(x)+\sigma-\mu_p(o)}{\sqrt{K_p(x,x)+K_p(o,o)-2K_p(x,o)}}\right)
\end{aligned}
\end{equation}
The other two terms are:
\begin{equation}
\label{eq:millsthird1142sec}
\begin{aligned}
&E[\nu(o)I_{\{|\nu(o)-\nu(x)|\leq \sigma\}}]=\mu_p(o)\\
&-E[\nu(o)I_{\{\nu(o)>\nu(x)+\sigma\}}]-E[\nu(o)I_{\{\nu(x)>\nu(o)+\sigma\}}]\\
&=\mu_p(o)-\mu_p(o)\left(1-\Phi\left(\tfrac{(\mu_p(x)+\sigma-\mu_p(o))}{\sqrt{K_p(x,x)+K_p(o,o)-2K_p(x,o)}}\right)\right)\\
& - \tfrac{K_p(o,o)-K_p(x,o)}{\sqrt {K_p(x,x)+K_p(o,o)-2K_p(x,o)}}\\
&\phi\left(\tfrac{\mu_p(x)+\sigma-\mu_p(o)}{\sqrt{K_p(x,x)+K_p(o,o)-2K_p(x,o)}}\right)\\
&-\mu_p(o)\left(1-\Phi\left(\tfrac{(-\mu_p(x)+\sigma+\mu_p(o))}{\sqrt{K_p(x,x)+K_p(o,o)-2K_p(x,o)}}\right)\right)\\
& + \tfrac{K_p(o,o)-K_p(x,o)}{\sqrt {K_p(x,x)+K_p(o,o)-2K_p(x,o)}}\\
&\phi\left(\tfrac{-\mu_p(x)+\sigma+\mu_p(o)}{\sqrt{K_p(x,x)+K_p(o,o)-2K_p(x,o)}}\right)\\
&=\mu_p(o)-\mu_p(o)\Bigg(2-\Phi\left(\tfrac{(\mu_p(x)+\sigma-\mu_p(o))}{\sqrt{K_p(x,x)+K_p(o,o)-2K_p(x,o)}}\right)\\
&-\Phi\left(\tfrac{(-\mu_p(x)+\sigma+\mu_p(o))}{\sqrt{K_p(x,x)+K_p(o,o)-2K_p(x,o)}}\right)\Bigg)\\
& + \tfrac{K_p(o,o)-K_p(x,o)}{\sqrt {K_p(x,x)+K_p(o,o)-2K_p(x,o)}}\\
&\Bigg(\phi\left(\tfrac{-\mu_p(x)+\sigma+\mu_p(o)}{\sqrt{K_p(x,x)+K_p(o,o)-2K_p(x,o)}}\right)\\
&-\phi\left(\tfrac{\mu_p(x)+\sigma-\mu_p(o)}{\sqrt{K_p(x,x)+K_p(o,o)-2K_p(x,o)}}\right)\Bigg)\\
\end{aligned}
\end{equation}
and
\begin{equation}
\label{eq:millsthird1142sec1}
\begin{aligned}
&E[\nu(x)I_{\{|\nu(o)-\nu(x)|\leq \sigma\}}]=\\
&=\mu_p(x)-\mu_p(x)\Bigg(2-\Phi\left(\tfrac{(\mu_p(x)+\sigma-\mu_p(o))}{\sqrt{K_p(x,x)+K_p(o,o)-2K_p(x,o)}}\right)\\
&-\Phi\left(\tfrac{(-\mu_p(x)+\sigma+\mu_p(o))}{\sqrt{K_p(x,x)+K_p(o,o)-2K_p(x,o)}}\right)\Bigg)\\
& + \tfrac{K_p(x,o)-K_p(x,o)}{\sqrt {K_p(x,x)+K_p(o,o)-2K_p(x,o)}}\\
&\Bigg(\phi\left(\tfrac{-\mu_p(o)+\sigma+\mu_p(x)}{\sqrt{K_p(x,x)+K_p(o,o)-2K_p(x,o)}}\right)\\
&-\phi\left(\tfrac{\mu_p(o)+\sigma-\mu_p(x)}{\sqrt{K_p(x,x)+K_p(o,o)-2K_p(x,o)}}\right)\Bigg)\\
\end{aligned}
\end{equation}

\paragraph{Proof of Proposition \ref{prop:3}}
If $\bR$ has \textbf{no uncertainty} and $\bH$ is rational, then the expected payoffs can be computed from \eqref{eq:DEFexp1}, \eqref{eq:IMMexp1} and \eqref{eq:OFFexp1}.
  The values are
$E[DEF]=\max(\mu_p(x),\mu_p(o))$, 
  $E[IMM]=\mu_p(x)$ and $E[DoN]=\mu_p(o)$. Therefore, DEF is always optimal.
  
   If $S$ is \textbf{bounded-rational} and $R$ has \textbf{no uncertainty}, then $\sigma>0$. We consider three cases:
   (1) $\mu_p(x)>\mu_p(o)+\sigma$;
   (2) $\mu_p(o)>\mu_p(x)+\sigma$;
   (3) otherwise.

In case (1), the payoffs are:
   $E[DEF]=\mu_p(x)$,
  $E[IMM]=\mu_p(x)$ and $E[DoN]=\mu_p(o)$,
   Therefore,    DEF is  optimal. A similar results holds in case (2). In case (3), $E[DEF]=\{\mu_p(x)-\epsilon,\mu_p(o)-\epsilon\}$. Under the condition (A) or (B),  DEF will alwys be dominated.
If $S$ is \textbf{rational} and $R$ has \textbf{uncertainty}, then 
  DEF is optimal as in Proposition \ref{prop:1}
The last case follows directly from Lemma \ref{lem:2}.

\paragraph{Proof of Proposition \ref{prop:honestr}}
The only case where the content of the message is important is when DEF is not optimal. In this case, $\bR$ makes a decision autonomously.

Whenever DEF is not optimal, the best action can be either IMM$(x)$ if $\mu_p(x) > \mu_p(o)$ or DoN if $\mu_p(o) > \mu_p(x)$. Therefore, if $\bH$ sends a biased message such that $\bR$ estimates $\mu_p(x) > \mu_p(o)$ when, in reality, $\nu(x) < \nu(o)$, then $\bR$ would choose an action that is not optimal for $\bH$.

\paragraph{Proof of Lemma \ref{lem:4}}
For large $s$, the beta distribution is approximatively a Gaussian distribution:
$$
{\displaystyle \operatorname {Beta} (w;t s,(1-t)s)\approx  {\mathcal {N}}\left(w;t,\frac{t(1-t) }{s}\right)}$$
Therefore, we have that
\begin{equation}
\label{eq:linearwproof}
\begin{aligned}
&p(x \succ y |\boldsymbol{\nu}(X))\\
&= \int I_{\{w \nu_1(x)+(1-w) \nu_2(x)>w \nu_1(y)+(1-w) \nu_2(y)\}}  {\mathcal {N}}\left(w;t,\frac{t(1-t) }{s}\right)dw\\
&= \int I_{\{w (\nu_1(x)-\nu_1(y)-(\nu_2(x)-\nu_2(y))>-(\nu_2(x)- \nu_2(y))\}}  {\mathcal {N}}\left(w;t,\frac{t(1-t) }{s}\right)dw\\
&= \int I_{\{(\tau z+t) (\nu_1(x)-\nu_1(y)-(\nu_2(x)-\nu_2(y))>-(\nu_2(x)- \nu_2(y))\}}  {\mathcal {N}}\left(z;0,1\right)dz\\
&= \int I_{\{\tau z (\nu_1(x)-\nu_1(y)-(\nu_2(x)-\nu_2(y))>-(\nu_2(x)- \nu_2(y))-t(\nu_1(x)-\nu_1(y)-(\nu_2(x)-\nu_2(y))\}}  {\mathcal {N}}\left(z;0,1\right)dz\\
&= \int I_{\{\tau z (\nu_1(x)-\nu_1(y)-(\nu_2(x)-\nu_2(y))>-(t\nu_1(x)+(1-t)\nu_2(x))+(t\nu_1(y)+(1-t)\nu_2(y))\}}  {\mathcal {N}}\left(z;0,1\right)dz\\
&= \Phi\left(\tfrac{(t\nu_1(x)+(1-t)\nu_2(x))-(t\nu_1(y)+(1-t)\nu_2(y))}{\tau  (\nu_1(x)-\nu_1(y)-(\nu_2(x)-\nu_2(y))}\right)I_{\{(\nu_1(x)-\nu_1(y)-(\nu_2(x)-\nu_2(y))>0\}} \\
&+ \Phi\left(-\tfrac{(t\nu_1(x)+(1-t)\nu_2(x))-(t\nu_1(y)+(1-t)\nu_2(y))}{\tau  (\nu_1(x)-\nu_1(y)-(\nu_2(x)-\nu_2(y))}\right)I_{\{(\nu_1(x)-\nu_1(y)-(\nu_2(x)-\nu_2(y))<0\}} \\ 
&= \Phi\left(\tfrac{(t\nu_1(x)+(1-t)\nu_2(x))-(t\nu_1(y)+(1-t)\nu_2(y))}{\tau  |\nu_1(x)-\nu_1(y)-(\nu_2(x)-\nu_2(y))|}\right)
\end{aligned}
\end{equation}
 where $\tau^2=\frac{t(1-t) }{s}$. 

\paragraph{Proof of Proposition \ref{prop:4}}

If $S$ is \textbf{rational} and $R$ has \textbf{no uncertainty}, then the payoff for IMM is if $\boldsymbol{\nu}(x)$, for DoN $\boldsymbol{\nu}(o)$, and for DEF is $\boldsymbol{\nu}(x)$ if it dominates $\boldsymbol{\nu}(o)$ (that is, if $\boldsymbol{\nu}(x)\succ \boldsymbol{\nu}(o)$) or $\boldsymbol{\nu}(o)$ otherwise. Therefore, DEF is not dominated.

If $S$ is \textbf{bounded-rational} and $R$ has \textbf{no uncertainty}, then the payoffs for IMM and DoN are the same as before. The payoff for DEF is $p\boldsymbol{\nu}(o)+(1-p)\boldsymbol{\nu}(x)$. So DEF is always dominated by either IMM  or DoN.

If $S$ is \textbf{rational} and $R$ has \textbf{ uncertainty}, then  DEF is never dominated. This results can be proven by Jensen's inequality (element-wise) as in Proposition \ref{prop:1}.

If $S$ is \textbf{bounded-rational} and $R$ has \textbf{ uncertainty}, the best decision depends on the specific case.

\paragraph{Proof of Proposition \ref{prop:context}}
Assume, by contradiction, that $\nu_1(a)$ does not depend on $a^*$. Since $\mathcal{A}$ is binary,  $\nu_1(a)$ (in $\nu([a,x]) = \nu_1(a)+ \nu_2(x)$) can only assume two values: $\nu_1(1)=\gamma_1$ and $\nu_1(0)=\gamma_0$.

When $a^*=1$, to satisfy D1,  $\min(\nu([1,o]),\nu([1,y]))> \max(\nu([0,x]))$. This implies that
$$
\gamma_1+\min(\nu_2(o),\nu_2(y))>  \gamma_0 + \nu_2(x).
$$
This is true, for every $o,x,y \in \mathcal{X}$, if
$$
\gamma_1> \gamma_0 + \max_{o,x,y \in \mathcal{X}}\left(\nu_2(x)-\min(\nu_2(o),\nu_2(y))\right)=\gamma_0+d
$$
with $d =\max_{o,x,y \in \mathcal{X}}\left(\nu_2(x)-\min(\nu_2(o),\nu_2(y))\right)\geq 0$.

When $a^*=0$, to satisfy D2,  $\min(\nu([0,o]),\nu([0,x]))>\max(\nu([1,y]))$. This implies that
$$
\gamma_0+\min(\nu_2(o),\nu_2(x))>  \gamma_1 + \nu_2(x).
$$
Therefore, we have that
$$
\gamma_0> \gamma_1 + \max_{o,x,y \in \mathcal{X}}\left(\nu_2(y)-\min(\nu_2(o),\nu_2(y))\right)=\gamma_1+d.
$$
Therefore, D1 and D2 cannot be both true, as 
$\gamma_1>\gamma_0+d>(\gamma_1+d)+d$ leads to a contradiction.

To prove the other direction, we can choose
$$
\nu([a,x]|a^*) = I_{\{a^*=0\}}\nu_0(a)+ I_{\{a^*=1\}}\nu_1(a)+\nu_2(x).
$$
Indeed, when $a^*=1$, we have that
$$
\gamma_{1}> \max_{o,x,y \in \mathcal{X}}\left(\nu_2(x)-\min(\nu_2(o),\nu_2(y))\right)= d,
$$
where $\gamma_{1}=\nu_1(1)$ and $\nu_1(0)=0$ (we have put it to zero withour loss of generality).
When $a^*=0$, we have that
$$
\gamma_{0}> \max_{o,x,y \in \mathcal{X}}\left(\nu_2(x)-\min(\nu_2(o),\nu_2(y))\right)= d,
$$
where $\gamma_{0}=\nu_0(1)$ and $\nu_0(0)=0$.

\bibliographystyle{elsarticle-num-names} 
 \bibliography{biblio}%

\end{document}